\documentclass{article}

\usepackage[preprint,nonatbib]{paper_style}
\usepackage[sort,comma,numbers]{natbib}

\usepackage[utf8]{inputenc} 
\usepackage[T1]{fontenc}    
\usepackage{hyperref}       
\usepackage{url}            
\usepackage{booktabs}       
\usepackage{amsfonts}       
\usepackage{nicefrac}       
\usepackage{microtype}      
\usepackage{xcolor}         
\usepackage{amsmath}
\usepackage{amssymb}
\usepackage{mathtools}
\usepackage{amsthm}
\usepackage{dsfont}
\usepackage{subcaption}
\usepackage{float}
\usepackage{cleveref}
\usepackage{algorithm}
\usepackage{algorithmic}
\usepackage{threeparttable}
\usepackage{multirow}

\DeclareMathOperator{\essinf}{ess\, inf}
\DeclareMathOperator{\esssup}{ess\, sup}

\theoremstyle{plain}
\newtheorem{theorem}{Theorem}
\newtheorem{proposition}{Proposition}
\newtheorem{lemma}{Lemma}
\newtheorem{assumption}{Assumption}
\newtheorem{corollary}{Corollary}
\theoremstyle{definition}

\theoremstyle{remark}
\newtheorem{remark}{Remark}

\usepackage{enumitem,amssymb}
\newlist{todolist}{itemize}{2}
\setlist[todolist]{label=$\square$}
\usepackage{pifont}
\makeatletter
\renewcommand{\paragraph}{%
  \@startsection{paragraph}{4}{\z@}%
                {\z@}%
                {-1em}%
                {\normalfont\normalsize\bfseries}}
\makeatother

\title{Doing well with less! On Sampling Techniques for Empirical Pairwise Loss Estimation/Minimization}

\author{%
  Louise Davy \\
  IDS, LTCI\\
  Télécom Paris\\
  Palaiseau, France\\
  \texttt{louise.davy@telecom-paris.fr} \\
  \And
  Stephan Clémençon \\
  IDS, LTCI\\
  Télécom Paris\\
  Palaiseau, France\\
  \texttt{stephan.clemencon@telecom-paris.fr} \\
  \And
  Charlotte Laclau \\
  IDS, LTCI\\
  Télécom Paris\\
  Palaiseau, France\\
  \texttt{charlotte.laclau@telecom-paris.fr} \\
}

\begin{document}


\maketitle

\begin{abstract}
Many machine learning problems, including similarity learning, ranking, and clustering, rely on empirical pairwise loss functions whose quadratic computational cost quickly becomes prohibitive at scale. We demonstrate how a frugal approach that retains only a fraction of the available information on pairs can achieve estimation or optimization performance comparable to that obtained by using all pairs, by leveraging survey sampling techniques. A central finding, supported by both theory and experiments, is that such sampling plans must target pairs directly rather than individual observations. In particular, for pairwise losses between high-dimensional vectors such as embeddings in vision or graph learning, assigning higher inclusion probabilities to informative pairs using suitable auxiliary information yields performance close to full pairwise evaluation, providing a principled and theoretically grounded trade-off between accuracy and computational cost.
\end{abstract}

\section{Introduction}
How to extract reliable information from a large population when measuring every individual is too costly? Long before the era of large-scale datasets, statisticians already faced this challenge. Survey sampling theory addressed it by designing principled schemes that select a small, informative subset of the population, and by correcting for the resulting selection bias in a statistically controlled way. The rapid growth of digital data and computational capacity temporarily shifted attention away from these ideas, as the focus moved toward scalability and the representation of complex data through learned embeddings. Yet randomization and sampling never truly left machine learning: from dropout and mini-batch SGD to boosting and random forests, careful data selection has remained central to efficient
learning. More recently, concerns about computational frugality and data efficiency have brought these questions back to the forefront.

A particularly challenging instance of this tension arises in tasks that rely on empirical pairwise loss functions. Given observations $Z_1, \ldots, Z_N \in \mathcal{Z}$, many statistical learning problems, including ranking, metric learning, clustering, and contrastive representation learning, are formulated as the minimization of a pairwise empirical risk of the form
\begin{equation}\label{eq:pairwise_loss}
    \frac{1}{N(N-1)} \sum_{i \neq j} \ell(\theta, (Z_i, Z_j)),
\end{equation}
where $\ell : \Theta \times \mathcal{Z}^2 \to \mathbb{R}$ is a pairwise loss function and $\Theta$ is the parameter space. Such objectives, referred to as U-statistics of order 2, explicitly depend on all $O(N^2)$ pairs in the dataset. When the $Z_i$ are high-dimensional, such as image or text embeddings, evaluating this objective over all pairs quickly becomes the dominant computational bottleneck. The most popular contrastive methods, such as 
SimCLR~\citep{chen2020simclr} and Barlow Twins~\citep{zbontar2021barlowtwins}, face exactly this challenge.

To reduce this cost, a natural strategy consists in subsampling individual observations and forming pairs from the resulting subset. The central contribution of this paper is to show that this strategy is suboptimal. A key principle of survey sampling is to leverage auxiliary information, that is, any cheap quantity correlated with the target, to guide the selection
of units. Applied to the population of pairs $\{(Z_i, Z_j) : i \neq j\}$, this means assigning higher inclusion probabilities to pairs expected to contribute more to the objective, using a proxy $\rho(Z_i, Z_j)$ correlated with the loss
$\ell(\theta, (Z_i, Z_j))$. Such auxiliary information can be defined directly at the pair level, for instance through edge weights in a graph, or derived from observation-level features, such as image metadata in face recognition datasets like LFW~\citep{haung2008lfw}. A simple yet practically important example is when $\ell(z, z')$ is an $\ell_q$ distance costly to compute in high dimension, while $\rho(z, z')$ is the distance of a low-dimensional projection, available at negligible cost. As we establish in Section~\ref{sec:estimation}, leveraging such auxiliary information
requires defining it at the pair level, which is precisely why the population of pairs, rather than that of individual observations, is the natural domain on which to apply a survey plan: the resulting Horvitz-Thompson estimator then achieves provably lower variance at equal computational cost.

\paragraph{Contributions.} This paper makes the following contributions. (i) (Sec.~\ref{sec:estimation}) We establish that, for any fixed pair budget, direct pair sampling strictly dominates observation sampling in terms of estimator variance. This holds for the most popular sampling designs used in practice, and is confirmed empirically across all settings considered.
(ii) (Sec .~\ref{sec:estimation}-\ref{sec:learning}) Setting inclusion probabilities proportional to a proxy $\rho(W_i, W_j)$ correlated with the loss produces a variance that depends on the quality of $\rho$. We carry this explicit dependence through to the generalization bound of Section~\ref{sec:learning}.
(iii) (Sec.~\ref{sec:learning})
We establish non-asymptotic excess risk bounds for predictors trained on survey-sampled pairs, for both Poisson and negatively associated designs. We further show that $\bar{n} = cN$ sampled pairs suffice to match the $\mathcal{O}_P(1/\sqrt{N})$ rate of full-pair learning, so that the ratio of pairs actually evaluated to pairs available vanishes as $N \to \infty$.

\paragraph{Related work.}

The use of survey sampling techniques to select training datasets for pointwise learning, that is, for minimizing empirical risk functionals of the form $(1/n)\sum_{i\leq n}l(\theta, Z_i)$, has received growing attention. Probabilistic tools for studying the performance of empirical risk minimizers under sampling without replacement have been developed in \cite{Hoeffding63} (see also \cite{BardenetMaillard}), and more general survey sampling schemes in \cite{CBC2017, pmlr-v63-clemencon64}. See also \cite{CLEM2015a} for the application of
survey techniques to mini-batch selection in SGD. However, the pairwise learning framework has not yet been documented in this context. The approximation of U-statistics by subsampling pairs has been studied in the statistics literature~\citep{Lee90}, and more recently in the machine-learning setting by \cite{CCB16}, where pairs are selected uniformly
at random. Unlike that approach, our framework leverages auxiliary information to define non-uniform inclusion probabilities over pairs sampled without replacement.

In the broader machine-learning literature, several lines of work address the selection of informative pairs for training, though from a fundamentally different perspective. Methods such as hard and semi-hard negative mining \citep{schroff2015facenet, robinson2020contrastive} and contrastive learning \citep{chen2020simclr, zbontar2021barlowtwins} select pairs heuristically based on the current model state, typically focusing on pairs with large loss or informative contrasts. Distance-weighted sampling \cite{wu2018samplingm} assigns non-uniform probabilities based on embedding distances. As shown in \cite{musgrave2020}, these approaches often fail to provide consistent gains across settings, and crucially, none of them corrects for the bias introduced by non-uniform selection.

Importance sampling offers another route to non-uniform pair selection: by reweighting gradient contributions proportionally to their norm, it aims to reduce variance in stochastic gradient descent \cite{zhao2015stochastic, katharopoulos2018not}. \cite{zhou2025} extend this to pairwise learning via PAC-Bayes analysis, noting that non-uniform sampling in pairwise learning remains unclear and has not been rigorously studied. Our work addresses this gap differently, by providing unbiased risk estimation and non-asymptotic concentration guarantees that are not available in importance sampling or contrastive learning approaches. Finally, ~\citep{pmlr-v139-bertail21a} uses survey sampling to debias samples from a shifted
source distribution, a setting orthogonal to ours.


\section{Background and preliminaries}\label{sec:background}
We recall the key notions of survey sampling theory and pairwise learning used throughout the paper.

\subsection{Survey sampling and Horvitz-Thompson estimation}
\label{sec:survey_sampling}

Consider a finite population of size $N$, indexed by $i \in \mathcal{I} := \{1, \ldots, N\}$. A probability distribution $D_N$ on $\mathcal{P}(\mathcal{I})$ is called a sampling plan on $\mathcal{I}$. The sampling outcome is encoded by $\boldsymbol{\epsilon}_N = (\epsilon_1, \ldots, \epsilon_N)$, where $\epsilon_i = \mathbb{I}\{i \in S\}$ for $1 \leq i \leq N$. Of particular interest are the first- and {second-order inclusion probabilities, defined for $i \neq j$ in $\mathcal{I}$ by $\pi_i = \mathbb{P}_{D_N}\{i \in S\}$ and $\pi_{i,j} = \mathbb{P}_{D_N}(\{i,j\} \subset S)$. A plan is said to be of \emph{fixed size} $n \leq N$ if $\#S = n$ almost surely, which implies $\sum_{i=1}^N \pi_i = n$.

\paragraph{Horvitz-Thompson estimation.} Suppose that a measurement
$z_i \in \mathcal{Z} \subset \mathbb{R}^d$ is assigned to each individual $i \in \mathcal{I}$. A central problem in survey theory is estimating the population total $Q = \sum_{i=1}^N z_i$ from a sample $S$. The Horvitz-Thompson (HT) estimator \citep{HT51} is the canonical tool
for this purpose:
\begin{equation}\label{eq:HT1}
    \widehat{Q}_{D_N} = \sum_{i \in S} \frac{z_i}{\pi_i} = \sum_{i=1}^N
    \frac{\epsilon_i}{\pi_i} z_i.
\end{equation}
It is unbiased, i.e. $\mathbb{E}_{D_N}[\widehat{Q}_{D_N}] = Q$, with variance given by \citet{yates53} for fixed-size plans:
\begin{equation}\label{eq:HT_condvar}
    \mathrm{Var}(\widehat{Q}_{D_N}) = \sum_{i < j} \left(\frac{z_i}{\pi_i} -
    \frac{z_j}{\pi_j}\right)^2 (\pi_i \pi_j - \pi_{i,j}).
\end{equation}
The auxiliary information used to define the inclusion probabilities typically takes the form of variables $w_i$ assigned to each individual, on which the $\pi_i$'s depend through a measurable link function $\rho : \mathcal{W} \to (0, +\infty)$, namely
$\pi_i = n\rho(w_i) / \sum_{j=1}^N \rho(w_j)$. The link function is chosen so as to make $\mathrm{Var}(\widehat{Q}_{D_N})$ as small as possible: ideally, inclusion probabilities proportional to the $z_i$'s minimize variance, though in practice a cheap proxy correlated with the $z_i$'s suffices.

\paragraph{Example 1: Poisson sampling.} The $\epsilon_i$'s are independent, and the plan is fully characterized by first-order inclusion probabilities
$\mathbf{p}_N = (p_1, \ldots, p_N) \in (0,1)^N$. The sample size is random with variance $d_N = \sum_{i=1}^N p_i(1-p_i)$, and second-order inclusion probabilities factorize as $\pi_{i,j} = p_i p_j$.

\paragraph{Example 2: Rejective sampling.} Rejective sampling $R_N$ is a fixed-size design of size $n \leq N$ that generalizes simple random sampling without replacement (SRSWOR). It can be viewed as Poisson sampling conditioned on achieving exactly $n$ draws, which is why it is also called \emph{conditional Poisson sampling} \citep{Dupacova}. The first-order inclusion probabilities $\pi_i$ of $R_N$ differ from the Poisson parameters $p_i$ used to construct it; their relationship is characterized in \cite{Hajek64}. Fixing the sample size reduces estimator variability compared to Poisson sampling, but introduces a dependency structure among the $\epsilon_i$'s. Specifically, the inclusion indicators are negatively associated \citep{joag1983negative}, a property shared by most fixed-size designs and which plays a central role in the concentration arguments of Section~\ref{sec:learning}. 

We provide additional examples of sampling plans in Appendix \ref{app:example-plans} and summarize their main characteristics in Table~\ref{tab:sampling_plans_full}.

\paragraph{Survey sampling and pointwise risk minimization.} As shown in
\cite{pmlr-v63-clemencon64}, empirical pointwise risk minimization can be extended to the case where a survey plan selects the training dataset. Placing ourselves in the superpopulation model, we assume that $\{(Z_i, W_i): i = 1, \ldots, N\}$ are independent copies of a pair $(Z, W)$, where $W$ encodes the auxiliary information driving the sampling design. The learning task consists in minimizing $L(\theta) = \mathbb{E}[l(\theta, Z)]$ over $\Theta$, based on a sample $S$ selected using the $W_i$'s. This naturally leads to minimizing the \emph{HT risk}:
\begin{equation}\label{eq:HTpointRisk}
    \widehat{L}_{D_N}(\theta) := \frac{1}{N} \sum_{i \in S} \frac{l(\theta, Z_i)}{\pi_i}.
\end{equation}
Upper confidence bounds of order $O_{\mathbb{P}}(1/\sqrt{n})$ for the excess risk $L(\widehat{\theta}_{D_N}) - \min_{\theta \in \Theta} L(\theta)$ have been established for Poisson designs in \cite{CBC2017} and for fixed-size designs in
\cite{pmlr-v63-clemencon64}. Our goal is to extend this framework to pairwise learning.

\subsection{\texorpdfstring{$U$}{U}-statistics and pairwise statistical learning}
\label{sec:ustat}

In pairwise learning, the goal is to minimize a risk of the form
$U(\theta) = \mathbb{E}[\ell(\theta, (Z, Z'))]$ over $\Theta$, where $Z$ and $Z'$ are independent copies of a random variable with unknown distribution $P$ on $\mathcal{Z}$, and $\ell : \Theta \times \mathcal{Z}^2 \to \mathbb{R}$ is a symmetric pairwise loss function. This formulation covers a wide range of tasks. In ranking, the objective is to find a scoring function $s : \mathcal{X} \to \mathbb{R}$ minimizing the ranking risk $\mathcal{R}(s) = \mathbb{E}[\mathbb{I}\{(s(X)-s(X'))(Y-Y') < 0\}]$, see \cite{CLV08}.
In \emph{metric learning}, the goal is to find a metric $\delta$ on  $\mathcal{X}$ such that pairs with the same label are close and pairs with different labels are far, formalized as minimizing $\mathcal{M}(\delta) = \mathbb{E}[\psi((\delta(X,X')-1)(2\mathbb{I}\{Y=Y'\}-1))]$ for a convex surrogate $\psi$, see \cite{BHS15}.

\paragraph{Empirical pairwise risk.} Given $N$ independent copies $Z_1, \ldots, Z_N$ of $Z$, the minimum-variance unbiased estimator of $U(\theta)$ is the pairwise average
\begin{equation}\label{eq:emp_pair_risk}
    \widehat{U}_N(\theta) = \frac{2}{N(N-1)} \sum_{1 \leq i < j \leq N}
    \ell(\theta, (Z_i, Z_j)),
\end{equation}
which is a U-statistic of order 2 with symmetric kernel $\ell(\theta, (z, z'))$,
see \cite{Lee90}. Concentration results for the U-process $\{\widehat{U}_N(\theta)\}_{\theta \in \Theta}$ can be established via Hoeffding
decompositions and decoupling \citep{PenaGine99, houdre03, CLV08}, yielding generalization bounds of order $O_{\mathbb{P}}(1/\sqrt{N})$ under standard complexity assumptions. However, evaluating \eqref{eq:emp_pair_risk} requires computing all $O(N^2)$ pairwise terms, which is prohibitive at scale. The following sections show how survey sampling provides principled grounded remedy.

\section[Survey sampling and U-statistics]{Survey sampling \& $U$-statistics}\label{sec:estimation}

 Placing ourselves in the superpopulation framework $(Z,W,\mathcal{I})$ previously recalled and setting $(\mathbf{Z}_N,\mathbf{W}_N)=((Z_1,\ldots,Z_N),(W_1,\ldots,W_N))$, this section focuses on leveraging survey sampling for empirical pairwise risk estimation with limited computation/memory capacities.  The theoretical arguments we develop here are also supported by solid empirical evidence. Throughout the section, we omit the argument $\theta$ to lighten notation and assume that $U=\mathbb{E}[\ell(Z,Z')]>0$.

\subsection{Sampled pairs \textit{vs} pairs of sampled observations} 

With a view to processing only a fraction of the high-dimensional data $Z$ from the large population $\mathcal{I}$ to estimate the empirical pairwise risk $\widehat{U}_N$, a first ``natural'' strategy consists in selecting randomly a sample $S$, of reduced (expected) size but as informative as possible, by means of an appropriate survey plan $D_N$ on $\mathcal{I}$ relying on the auxiliary information $\mathbf{W}_N$. Based on the sampled data $\{Z_i:\; i\in S\}$, an estimator of \eqref{eq:emp_pair_risk} is then
\begin{equation}
\label{eq:emp_pair_risk_on_obs}
\widetilde{U}_{D_N}=\frac{2}{N(N-1)}\sum_{i<j\; \text{ in } S}\frac{\ell(Z_i,Z_j)}{\pi_{i,j}}\\=\frac{2}{N(N-1)}\sum_{1\leq i<j\leq N}\frac{\epsilon_i\epsilon_j}{\pi_{i,j}}\ell(Z_i,Z_j).
\end{equation}
Note that to ensure zero bias, it is necessary to weight by the inverse of the 2nd order inclusion probabilities in this case, \textit{i.e.} $\mathbb{E}[\epsilon_i\epsilon_j]=\pi_{i,j}$. In the Poisson case however, the latter are expressed as a function of the 1st order inclusion probabilities (\textit{i.e.} $p_{i,j}=p_ip_j$) and the conditional variance of \eqref{eq:emp_pair_risk_on_obs} given the $W_i$'s is then given by
\begin{multline}\label{eq:var_Poisson_obs}
\binom{N}{2}^{-2} Var(\widetilde{U}_{P_N}\mid \mathbf{W}_N)=\sum_{i<j}\frac{1-p_ip_j}{p_ip_j}\ell^2(Z_i,Z_j) 
+\frac{1}{2}\sum_{i=1}^N\sum_{\substack{j\neq i\\ k\neq i\\ j<k}} \frac{1 -p_i}{p_i} \ell(Z_i,Z_j)\ell(Z_i,Z_k),
\end{multline}
omitting to specify the dependence of the $p_i$'s in $\mathbf{W}_N$ when the latter is fixed for simplicity.
The conditional expectation of the number of pairs $(Z_i,Z_j)$ actually used in the computation given $\mathbf{W}_N$ is $\mathbb{E}[\sum_{i<j}\epsilon_i\epsilon_j\mid \mathbf{W}_N]=\sum_{i<j}p_ip_j:=\tilde{n}$.

\noindent {\bf Sampling pairs of observations.} Alternatively, rather than sampling individual observations and forming next pairs of sampled observations, one could consider sampling directly pairs of indices by means of a survey scheme $\bar{D}_{\bar{N}}$ in the population $\mathcal{J}:=\{(i,j):\; 1\leq i<j\leq N\}$ of cardinality $\bar{N}=N(N-1)/2$ based on information $\mathbf{W}_N$, yielding a sample of pairs $\bar{S}$ and the estimator of \eqref{eq:emp_pair_risk}:
\begin{equation}
\label{eq:emp_pair_risk_on_pairs}
\bar{U}_{\bar{D}_{\bar{N}}}=\frac{2}{N(N-1)}\sum_{(i,j) \in \bar{S}}\frac{\ell(Z_i,Z_j)}{\bar{\pi}_{i,j}}\\
=\frac{2}{N(N-1)}\sum_{i<j}\frac{\bar{\epsilon}_{i,j}}{\bar{\pi}_{i,j}}\ell(Z_i,Z_j),
\end{equation}
with $\bar{\pi}_{i,j}(\mathbf{W}_N)=\mathbb{P}(\bar{\epsilon}_{i,j}=1\mid \mathbf{Z}_N,\mathbf{W}_N)$ and $\bar{\epsilon}_{i,j}=\mathbb{I}\{(i,j)\in \bar{S}\}$ for $i<j$.
For more clarity, we provide table \ref{tab:reminder_notations} in the Appendix to remind the different notations between observation and pair sampling. Conditioned upon $(\mathbf{Z}_N,\mathbf{W}_N)$, \eqref{eq:emp_pair_risk_on_pairs} is a HT estimator, whose variance can be deduced from the Sen-Yates-Grundy formula \eqref{eq:HT_condvar} when the pairwise plan $\bar{D}_{\bar{N}}$ is of fixed size and is equal to
\begin{equation}
Var(\bar{U}_{\bar{P}_{\bar{N}}}\mid \mathbf{Z}_N,\mathbf{W}_N)=\\
\frac{4}{N(N-1)^2}\sum_{i<j}\left(\frac{1}{\bar{p}_{i,j}}-1 \right)\ell^2(Z_i,Z_j)
\end{equation}
when it is a pairwise Poisson scheme $\bar{P}_{\bar{N}}$ with inclusion probabilities $\bar{p}_{i,j}=\bar{p}_{i,j}(\mathbf{W}_N)$ for $i<j$ and conditional expected size $\bar{n}=\sum_{i<j}\bar{p}_{i,j}$. Reminders on notation for pair- and observation-level sampling plans are collected in Appendix~\ref{app:notations}.
Finding an explicit expression for the inclusion probabilities that minimizes the above quantity under the constraint that the expected number of pairs used  $\sum_{i<j}\bar{p}_{i,j}$ is fixed, equal to $\bar{n}\in (0,\bar{N})$ with probability $1$ requires the following assumption.
\begin{assumption}\label{hyp1}
Fix $\bar{n}\in (0,\bar{N})$. We have:
\begin{equation*}
0<\bar{n}\esssup \rho^*(W,W')\leq \bar{N}\essinf \rho^*(W,W')<+\infty,
\end{equation*}
where $\rho^*(W,W')=\sqrt{\mathbb{E}[\ell^2(Z,Z')\mid W,W']}$.
\end{assumption}

The result below describes the inclusion probabilities of the optimal pairwise Poisson plan with a given expected size.
\begin{theorem}\label{thm:opt_pair}{\sc (Optimal pairwise Poisson scheme)} 
Fix $\bar{n}\in (0,\bar{N})$ and suppose that Assumption \ref{hyp1} is fulfilled. The pairwise Poisson scheme solution to the minimization problem under size constraint
\begin{equation}
\min_{\bar{P}_{\bar{N}}}\mathbb{E}[Var(\bar{U}_{\bar{P}_{\bar{N}}}\mid \mathbf{Z}_N,\mathbf{W}_N)\mid \mathbf{W}_N]
\text{ subject to } \sum_{i<j}\bar{p}_{i,j}=\bar{n}
\end{equation}
is the pairwise Poisson plan $\bar{P}^*_{\bar{N}}$ defined by the link function $\rho^*$,
namely that with inclusion probabilities given by $\bar{p}^*_{i,j}(\mathbf{W}_N)=\bar{n}\rho^*(W_i,W_j)/\sum_{k<l}\rho^*(W_k,W_l)$ for $i<j$. For $\bar{D}_{\bar{N}}=\bar{P}^*_{\bar{N}}$, the (unconditional) variance of the estimator \eqref{eq:emp_pair_risk_on_pairs} of $U=\mathbb{E}[\ell(Z,Z')]$, minimal among all HT estimators based on a pairwise Poisson scheme with fixed size $\bar{n}$, is given by $Var(\bar{U}_{\bar{P}^*_{\bar{N}}})=Var(\widehat{U}_N)+\bar{\sigma}^2_*$, where
\begin{equation*}
\bar{\sigma}^2_*= \frac{\bar{N}-1}{\bar{N}\bar{n}}\left(\mathbb{E}[\rho^*(W,W')]\right)^2-\frac{\bar{n}-1}{\bar{N}\bar{n}}\mathbb{E}\left[(\rho^*(W, W'))^2\right].
\end{equation*}
\end{theorem}
Refer to the Appendix for the technical proof.
Observe that $\mathbb{E}[ \ell^2(Z,Z')]\geq U^2>0$ and define $\psi(W)=\mathbb{E}[ \ell^2(Z,Z')\mid W ]/\sqrt{\mathbb{E}[ \ell^2(Z,Z')]}$. Consider the event 
$$
\mathcal{E}_{\ell}=\left\{\mathbb{E}\left[ \ell^2(Z,Z')\mid W,W' \right]= \psi(W)\psi(W')\right\}.
$$
The hypothesis below basically stipulates that the function describing the auxiliary information provided by the pair $W,W'$ on the quantity $\ell^2(Z,Z')$ cannot be expressed as a tensor product of a function of $W$ and a function of $W'$. In particular, it prohibits $\ell(Z,Z')$ from being expressed as $l(Z)l(Z')$ for some measurable function $l:\mathcal{Z}\to \mathbb{R}_+$. 
\begin{assumption}\label{hyp2}
The event $\mathcal{E}_{\ell}$ does not occur with probability one.
\end{assumption}

The following result shows that no estimator of the form \eqref{eq:emp_pair_risk_on_obs} based on a Poisson plan $P_N$ on $\mathcal{I}$ with an expected number of pairs sampled equal to $\sum_{i<j}p_ip_j=\bar{n}$ can achieve an expected conditional variance as low as that of $\bar{U}_{\bar{P}^*_N}$.
\begin{proposition}\label{prop:gain}
Fix $\bar{n}\in (0,\bar{N})$ and suppose that Assumption \ref{hyp1} is fulfilled. Let $P_N$ any Poisson scheme on $\mathcal{I}$ based on information $\mathbf{W}_N$ with inclusion probabilities $p_i$ s.t. $\sum_{i<j}p_ip_j=\bar{n}$. We almost-surely have:
\begin{equation*}
\mathbb{E}[Var(\widetilde{U}_{P_N}\mid \mathbf{Z}_N,\mathbf{W}_N)\mid \mathbf{W}_N]\geq\\
\mathbb{E}[Var(\bar{U}_{\bar{P}^*_{\bar{N}}}\mid \mathbf{Z}_N,\mathbf{W}_N)\mid \mathbf{W}_N].
\end{equation*}
If, in addition, Assumption \ref{hyp2} holds true, then
\begin{equation*}
\mathbb{E}[Var(\widetilde{U}_{P_N}\mid \mathbf{Z}_N,\mathbf{W}_N)]>
\bar{\sigma}^2_*.
\end{equation*}
\end{proposition}
The proof follows from the arguments demonstrating the previous theorem combined with \eqref{eq:var_Poisson_obs} (see Appendix for details). In summary, for the same auxiliary information $\mathbf{W}_N$ and
expected number $\bar{n}$ of evaluated pairs, only direct pair
sampling achieves optimal accuracy in the least squares sense. 

We point out that the result can be extended to fixed-size survey sampling plans, provided that the 2nd-order inclusion probabilities can be approximated based on those of the 1st order, given the expression of conditional variance in this case. In particular, this is true for conditional Poisson (rejective) schemes. Due to space constraints, this extension is postponed to the Appendix.

In practice, of course, the optimal pairwise link function $\rho^*(w,w')$ is generally unknown. As illustrated in the following subsection, it is often sufficient to use a pairwise Poisson plan $\bar{P}_{\bar{N}}$ defined by a link function $\rho(w,w')$ such that $\rho(W,W')$ and $\ell(W,W')$ tend to vary in the same direction to approximate the quantity \eqref{eq:emp_pair_risk} by $\bar{U}_{\bar{P}_{\bar{N}}}$, which has expected conditional variance converging to
\begin{equation}\label{eq:sigma_rho}
\sigma^2_{\rho}:=\frac{\mathbb{E}[\rho(W,W')]}{\bar{n}}\mathbb{E}\left[\frac{\ell^2(Z,Z')}{\rho(W,W')}\right], \qquad \text{as $\bar{N}$ becomes asymptotically large.}
\end{equation}

\subsection{Estimation experiments}

\paragraph{Numerical illustration}

\begin{figure}[t]
    \centering
    \begin{subfigure}[b]{0.3\textwidth}
        \centering
\includegraphics[width=\textwidth]{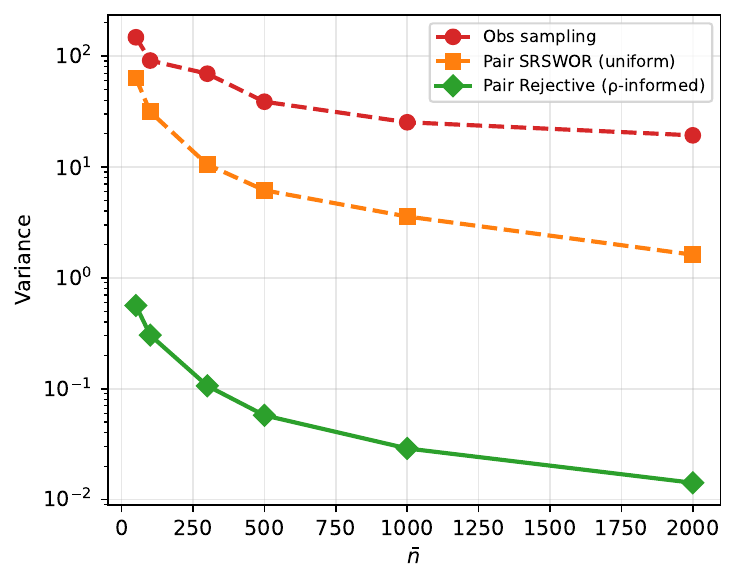}
        \caption{Rejective}
    \label{fig:toy_variance}
    \end{subfigure}
    \hfill
     \begin{subfigure}[b]{0.3\textwidth}
        \centering
\includegraphics[width=\textwidth]{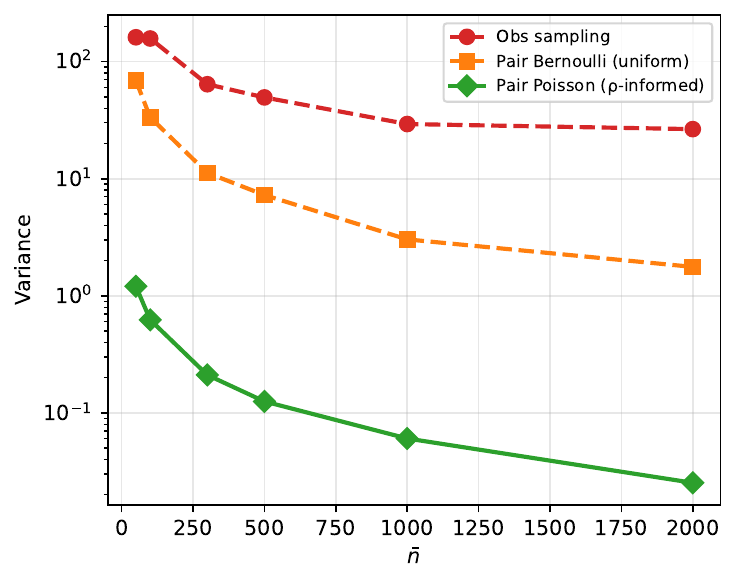}
        \caption{Poisson}
 \label{fig:toy_variance_poisson}
    \end{subfigure}
    \hfill
    \begin{subfigure}[b]{0.3\textwidth}
        \centering
\includegraphics[width=\textwidth]{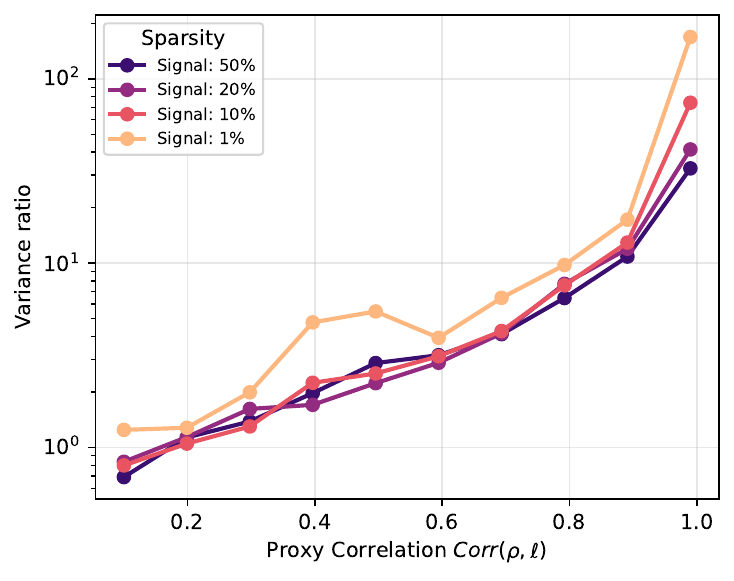}
\caption{Impact of $\rho$}
\label{fig:toy_sensitivity}
    \end{subfigure}
\caption{Illustration of the HT-risk estimator efficiency. 
(a) and (b) Comparison of the variance of $\widehat{R}_{\pi}(\theta)$ for a sparse pairwise task. (c) Reduction of variance gain as a function of the correlation between the auxiliary information $\rho$ and the loss $\ell$.}
\label{fig:toy_example_results}
\end{figure}
We evaluate the HT-risk estimator $\hat{R}(\theta)$ on a synthetic task where the signal is sparse: only $0.9\%$ of pairs contribute to $93\%$ of the total risk $R(\theta)$. We compare our approach using a $\rho$-informed rejective plan against two baselines: uniform pair sampling (SRSWOR) and standard observation sampling. Figure~\ref{fig:toy_variance} shows that our method achieves a variance reduction of up to $121\times$ over uniform pair sampling, without evaluating all $O(N^2)$ pairs. Furthermore, Figure~\ref{fig:toy_sensitivity} demonstrates the robustness of this approach: the efficiency gain grows with $\mathrm{Corr}(\rho, \ell)$, confirming that the theoretical benefits of Theorem~\ref{thm:opt_pair} are accessible in practical, noisy settings. 

\paragraph{Real-world datasets.}
The three tasks below are representative of standard pairwise learning problems, where evaluating the pairwise loss over all pairs is the dominant computational bottleneck.
\textit{MovieLens} ($N=1\,683$) targets preference asymmetry between items via a BPR-style pairwise loss, as in collaborative filtering and learning-to-rank. The auxiliary information $\rho(W_i, W_j)$ is derived from item popularity and mean ratings, metadata available before any pairwise computation. \textit{Cora} ($N=2\,708$) models node classification as a pairwise hinge loss over node feature embeddings, a setting encountered in graph contrastive learning. Here $\rho(W_i, W_j)$ combines a low-dimensional projection of node features and common neighbors, naturally available without pairwise feature computation.
\textit{LFW} ($N=13\,233$) addresses metric learning for face verification, optimizing a pairwise FAR-based loss over image embeddings obtained from a pretrained FaceNet model \cite{schroff2015facenet}. With over 87 million pairs, this is the largest dataset considered. As auxiliary information $\rho(W_i, W_j)$, we use the visual attributes of \citet{kumar2009attribute} most correlated with elevated false acceptance rates. In all three cases, $\rho(W_i, W_j)$ is available at negligible cost. It requires no additional supervision beyond what is already available in standard pipelines, and is positively correlated with the pairwise loss.
\begin{table}
\centering
\caption{
    Variance reduction at $\bar{n} = N$ (the regime of Corollary~1)
    across datasets and sampling designs.
    \emph{Bern.\,/\,Pois.} and \emph{SRS\,/\,Rej.} report the ratio of
    estimator variance between uniform and informed pair sampling.
    \emph{Obs.\,/\,Pair} reports the ratio between observation-based
    and uniform pair sampling.
    All ratios are $\geq 1$; larger values indicate greater variance reduction.
}
\label{tab:variance_reduction}
\resizebox{\linewidth}{!}{
\begin{tabular}{@{}llccrrrr@{}}
\toprule
Dataset & Task & $\mathrm{CV}^2(\ell)$ & $\mathrm{Corr}(\rho, \ell)$
        & Bern.\,/\,Pois.\,$(\uparrow)$
        & SRS\,/\,Rej.\,$(\uparrow)$
        & Obs.\,/\,Pair\,$(\uparrow)$ \\
\midrule
Toy
  & Sparse signal 
  & $99.8$ & $0.97$
  & $64.0$x &$105.0$x  & $9.0$x \\
MovieLens
  & Ranking loss
  & $30.8$ & $0.51$
  & $3.1$x &$3.8$x & $42.0$x \\
Cora
  & Pairwise hinge loss
  & $8.1$ & $0.36$
  & $2.0$x & $2.4$x& $40.2$x \\
LFW
  & FAR estimation 
  & $3.98$ & $0.16$
  & $1.2$x & $2.5$x & $76.8$x \\

\bottomrule
\end{tabular}
}
\end{table}

\paragraph{Results.}
Table \ref{tab:variance_reduction} reports variance ratios at $\bar n = N$ across datasets and sampling designs. 

\textit{Informed vs. uniform pair sampling}. As predicted by Theorem~\ref{thm:opt_pair} and confirmed by Figure~\ref{fig:toy_sensitivity}, the gain from informed sampling increases with the quality of the proxy $\rho$. The variance is reduced from $105\times$ for the sparse toy task where $\rho$ is nearly a perfect proxy, down to $2-3.8 \times$ on MovieLens and Cora where the correlation $\mathrm{Corr}(\rho, \ell)$ is moderate, and $1.2-2.5\times$ on LFW where it is weak. Note that the gain is systematically smaller for Bernoulli/Poisson than for SRS/Rejective, as the variance of Poisson-based estimators is partly dominated by sample size variability rather than pair selection, an effect documented in Appendix.

\textit{Pair-level vs. observation-level sampling.} The Obs./Pair column reflects the result established in Proposition~\ref{prop:gain}: pair-level sampling dominates observation-level sampling regardless of $\rho$, with a gain that reflects the intrinsic heterogeneity of the loss, measured by its squared coefficient of
variation $\text{CV}^2(\ell) := \text{Var}(\ell)/\mathbb{E}[\ell]^2$. LFW illustrates this most starkly (76.8× reduction despite a weak proxy). MovieLens and Cora confirm the trend at intermediate heterogeneity levels.  This also implies that no Poisson scheme on observations, however well-informed, can match the efficiency of direct pair sampling, because observation-level inclusion probabilities are constrained to factorize as $p_i p_j$ and cannot adapt to the full geometry of the pairwise loss. These results are further illustrated in Figure~\ref{fig:supp-cora-movielens} in Appendix~\ref{app:additional-results}, which displays MAE and variance as a function of $\bar{n}$ for Cora, MovieLens and LFW.

The estimation results above show that survey-sampled pairs yield provably lower variance than both observation-based and uniform pair sampling. We now ask whether this variance reduction translates into better generalization when the sampled pairs are used to minimize the pairwise risk.


\section{Pairwise learning based on survey-sampled pairs}\label{sec:learning}
We now move beyond loss estimation and investigate the statistical performance of learning procedures trained on survey-sampled pairs. We consider learning by minimizing the survey-weighted empirical pairwise risk based on a subsample of pairs drawn according to a sampling design $\bar{D}_{\bar{N}}$:
$
    \bar{\theta}_{\bar{D}_{\bar{N}}} \in \arg\min_{\theta \in \Theta}\;
    \bar{U}_{\bar{D}_{\bar{N}}}(\theta),
$
where $\bar{U}_{\bar{D}_{\bar{N}}}(\theta)$ denotes the HT estimator of the empirical
pairwise risk introduced in \eqref{eq:emp_pair_risk_on_pairs}. We place ourselves in the superpopulation
framework and study the deviation between the survey-weighted empirical pairwise risk and
the true pairwise risk.

\subsection{Generalization guarantees}
\label{sec:generalization}

For any $\theta \in \Theta$, the excess risk admits the decomposition
\begin{equation}\label{eq:decomp}
    \bar{U}_{\bar{D}_{\bar{N}}}(\theta) - U(\theta) =
    \underbrace{\bar{U}_{\bar{D}_{\bar{N}}}(\theta) - \widehat{U}_N(\theta)}_{\text{survey-induced error}}
    + \underbrace{\widehat{U}_N(\theta) - U(\theta)}_{\text{U-statistic error}},
\end{equation}
where $\widehat{U}_N(\theta)$ is the complete U-statistic based on all pairs. Uniform deviation bounds for the second term are available from existing results on U-statistics \citep{CLV08}. The first term captures the additional error induced by the survey sampling of pairs and constitutes the main object of our analysis. We make the following assumptions throughout this section.
\begin{enumerate}[label=(A\arabic*)]
    \item \label{ass:bounded} \textbf{(Bounded loss)} $|\ell_\theta(z, z')| \leq M$ for
    all $\theta \in \Theta$ and all $(z, z') \in \mathcal{Z}^2$.
    \item \label{ass:lipschitz} \textbf{(Lipschitz loss)} There exists $L > 0$ such that
    $|\ell_\theta(z,z') - \ell_{\theta'}(z,z')| \leq L\|\theta - \theta'\|$ for all
    $\theta, \theta' \in \Theta$ and $(z,z') \in \mathcal{Z}^2$.
    \item \label{ass:incl} \textbf{(Non-degenerate inclusion probabilities)}
    $\bar{\pi}_{i,j} \geq \bar{\pi}_{\min} > 0$ for all $(i,j) \in \mathcal{J}$.
\end{enumerate}

The following theorem provides a uniform deviation bound for the survey-induced error, valid for both Poisson and negatively associated designs. Recall that under a pairwise plan with link function $\rho$ satisfying Assumption~\ref{hyp1}, the inclusion probabilities are $\bar{\pi}_{i,j} = \bar{n}\rho(W_i,W_j)/S^\rho_N$ where $S^\rho_N = \sum_{k < l}\rho(W_k, W_l)$. We define the variance proxy
\begin{equation}\label{eq:Vrho}
    V^\rho := M^2 \sum_{i < j} \left(\frac{S^\rho_N}{\bar{n}\,\rho(W_i, W_j)} - 1\right),
\end{equation}
which quantifies the concentration of inclusion probabilities relative to the link function $\rho$: the more $\rho$ is aligned with the loss, the smaller $V^\rho$.

\begin{theorem}[Uniform deviation under pair sampling]
\label{thm:main-learning}
Assume \ref{ass:bounded}--\ref{ass:incl} and consider a pairwise sampling design
$\bar{D}_{\bar{N}}$ satisfying either \emph{(i)} independence of the
$(\bar{\epsilon}_{i,j})$ \emph{(Poisson)}, or \emph{(ii)} negative association of the
$(\bar{\epsilon}_{i,j})$ \emph{(rejective, and more generally any fixed-size design)}.
Let $\Theta_\eta$ be an $\eta$-net of $\Theta$ for $\|\cdot\|$, with covering number
$\mathcal{N}(\Theta, \|\cdot\|, \eta)$. Then for any $\eta > 0$ and $\delta \in (0,1)$,
with probability at least $1 - \delta$,
\begin{equation}\label{eq:main_bound}
    \sup_{\theta \in \Theta} \left|\bar{U}_{\bar{D}_{\bar{N}}}(\theta) -
    \widehat{U}_N(\theta)\right| \leq \frac{1}{\bar{N}}\left[
    \sqrt{2V^\rho \log\frac{2\mathcal{N}(\Theta,\|\cdot\|,\eta)}{\delta}}
    + \frac{2B}{3}\log\frac{2\mathcal{N}(\Theta,\|\cdot\|,\eta)}{\delta}
    \right] + \frac{L}{\bar{\pi}_{\min}}\eta,
\end{equation}
where $B = M/\bar{\pi}_{\min}$, and $\bar{n}$ denotes the \emph{expected} number of sampled pairs under Poisson sampling, and the \emph{exact} sample size under rejective sampling. 
\end{theorem}
The proofs for cases \emph{(i)} and \emph{(ii)} are given in Appendices~\ref{app:sec4-poisson} and~\ref{app:sec4-rejective} respectively, using Bernstein's inequality for independent and negatively associated variables. The bound \eqref{eq:main_bound} makes explicit how the choice of link function $\rho$ controls the tightness of the guarantee through $V^\rho$. In particular, the uniform plan $\rho \equiv 1$ recovers a generic constant, while any informative $\rho$ yields a strictly smaller $V^\rho$. This mirrors the optimality structure of Theorem~\ref{thm:opt_pair}: the same link
function $\rho^*$ that minimizes estimation variance also minimizes $V^\rho$, so that better auxiliary information simultaneously reduces variance and tightens generalization guarantees.
Combining \eqref{eq:main_bound} with standard U-process bounds \citep{CLV08}, the excess risk of the learned predictor satisfies
\begin{equation}\label{eq:excess_risk}
    U(\bar{\theta}_{\bar{D}_{\bar{N}}}) - \inf_{\theta \in \Theta} U(\theta) =
    O_{\mathbb{P}}\!\left(\sqrt{\frac{\log \mathcal{N}(\Theta)}{\bar{n}}}\right)
    + O_{\mathbb{P}}\!\left(\sqrt{\frac{\log \mathcal{N}(\Theta)}{N}}\right),
\end{equation}
where the first term is the survey-induced error and the second is the standard U-statistic approximation error.

\begin{corollary}[Sample complexity]\label{cor:sample}
Let $c > 0$ and suppose $\bar{n} = cN$. Then both terms in \eqref{eq:excess_risk} are of order $O_{\mathbb{P}}(1/\sqrt{N})$, and the excess risk satisfies
\begin{equation}
        U(\bar{\theta}_{\bar{D}_{\bar{N}}}) - \inf_{\theta \in \Theta} U(\theta) \leq
        \left(\frac{C_1}{\sqrt{c}} + C_2\right)\sqrt{\frac{\log \mathcal{N}(\Theta)}{N}},
\end{equation}
with high probability, where $C_1$ and $C_2$ are constants depending only on $M$, $\bar{\pi}_{\min}$, and $L$. In particular, survey sampling with $\bar{n} = cN$ pairs matches the $O_{\mathbb{P}}(1/\sqrt{N})$ rate of full-pair learning, while reducing the computational cost from $O(N^2)$ to $O(N)$ pairs. 
\end{corollary}

In particular, the reduction factor $N/(2c)$ grows unboundedly with $N$, 
reflecting the quadratic-versus-linear contrast between the two sampling regimes. We next assess empirically how survey-sampled pairwise training behaves in practice.

\subsection{Learning experiments}
\paragraph{Node Classification}
We train a two-layer GCN \citep{gcn} on Cora with a pairwise hinge loss and evaluate node classification accuracy via logistic regression on the learned embeddings. We compare full-pair training (oracle), Bernoulli sampling, Poisson sampling, and a method inspired by hard negative mining but adapted to supervised tasks we call hard sampling (see Appendix \ref{appendix:experiment_settings}).

\begin{figure}
    \centering
    \begin{subfigure}[b]{0.32\textwidth} \includegraphics[width=\linewidth]{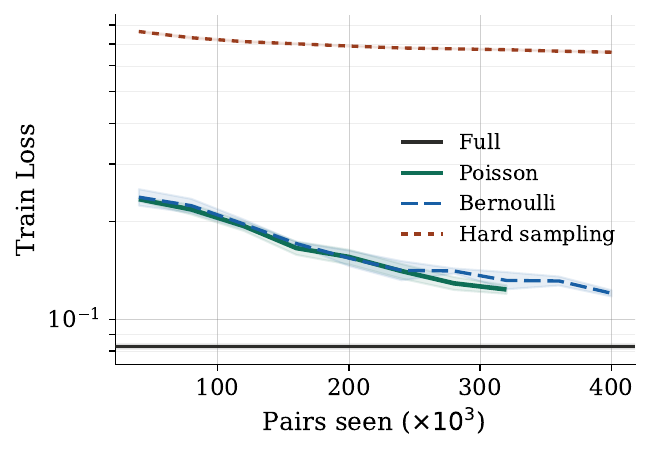}
        \caption{}
        \label{fig:loss-cora}
    \end{subfigure}
    \begin{subfigure}[b]{0.32\textwidth} \includegraphics[width=\linewidth]{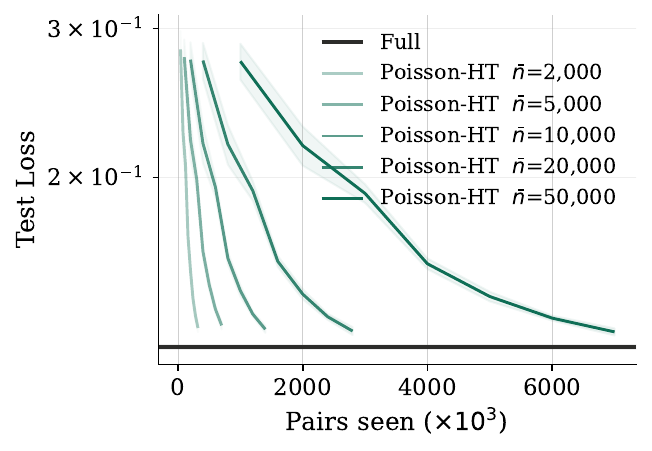}
        \caption{}
        \label{fig:budget-cora}
    \end{subfigure}
     \begin{subfigure}[b]{0.32\textwidth} \includegraphics[width=\linewidth]{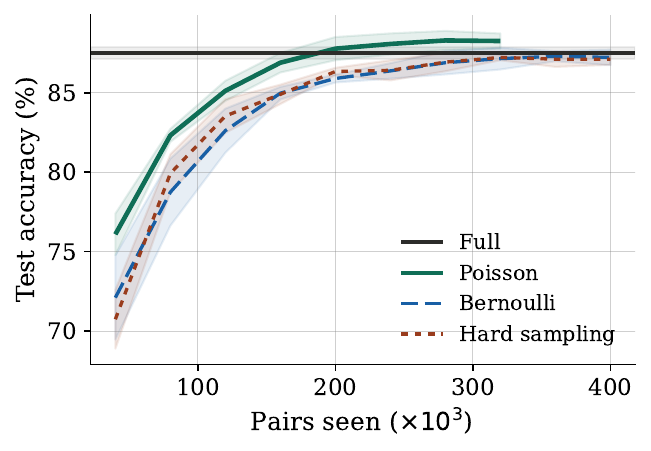}
        \caption{}
        \label{fig:acc-cora}
    \end{subfigure}\\
    \begin{subfigure}[b]{0.32\textwidth} \includegraphics[width=\linewidth]{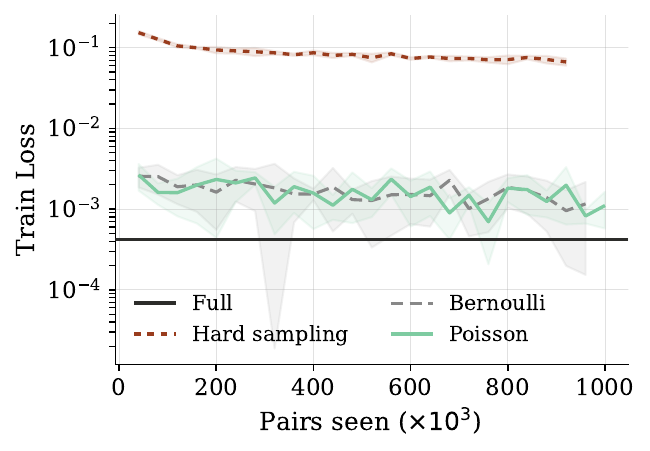}
        \caption{}
        \label{fig:loss-lfw}
    \end{subfigure}
    \begin{subfigure}[b]{0.32\textwidth} \includegraphics[width=\linewidth]{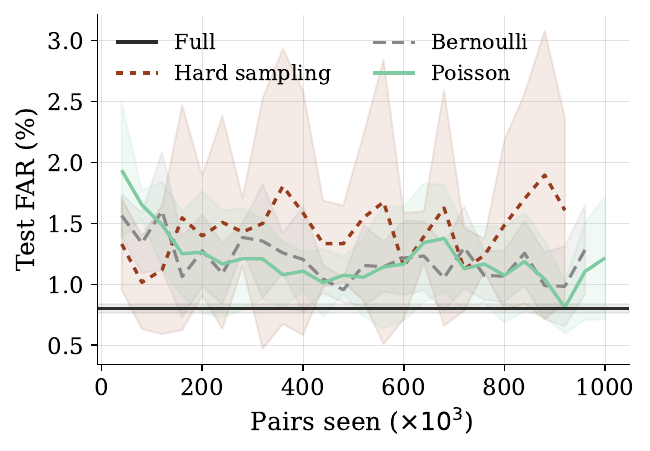}
        \caption{}
        \label{fig:far-lfw}
    \end{subfigure}
     \begin{subfigure}[b]{0.32\textwidth} \includegraphics[width=\linewidth]{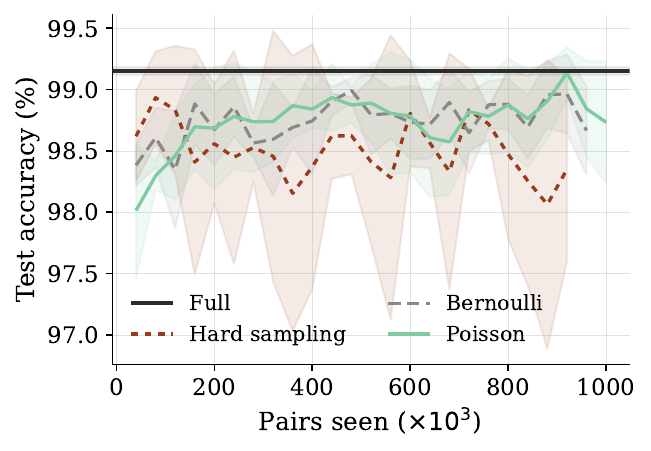}
        \caption{}
        \label{fig:acc-lfw}
    \end{subfigure}
    \caption{Results on Cora (a, b, c) and LFW (d, e, f). Results are averaged over 5 seeds, shaded = $\pm$std). (a, d) At budget $\bar n= 2000$ per epoch (over 150 epochs). (b) All Poisson test losses as the number of pairs seen increases, with different number of sampled pairs per epoch. (c, f) Best test accuracy and (e) FAR as a function of the number of pairs seen. The full baselines were trained on \textasciitilde 3 million pairs for Cora, and \textasciitilde 80 millions for LFW.}
    \label{fig:cora-learning}
\end{figure}
\paragraph{Face Recognition} We train a siamese contrastive model \cite{hadsell06} on LFW and evaluate it on the face recognition task, using the same protocol as for node classification: full-pair training, Poisson, Bernoulli, and hard sampling are compared against the full-pair oracle. We consider two types of auxiliary information to define the inclusion probabilities of the pairs : (i) the top-10 face attributes most correlated with the false-accept rate, as an interpretable, annotation-based proxy for pair difficulty; (ii) the contrastive loss itself, evaluated on pre-trained FaceNet embeddings, which serves as an oracle setting where auxiliary information is nearly perfectly correlated with the target loss.

\paragraph{Results}
At budget $\bar{n} = 2{,}000$, Poisson and Bernoulli both converge toward the full-pair loss while hard sampling plateaus due to its uncorrected selection bias
(Figures~\ref{fig:loss-cora},~\ref{fig:loss-lfw}). Figure~\ref{fig:budget-cora} confirms that Poisson sampling approaches the full-pair baseline across all
budgets (Corollary~1), with larger budgets closing the remaining gap at the cost of more computation. Figures~\ref{fig:acc-cora} and~\ref{fig:acc-lfw} further
show that Poisson sampling remains competitive while requiring substantially fewer pairs. In face verification, false acceptance, that is incorrectly granting access to an impostor, is the more consequential error type from a security standpoint;
Figure~\ref{fig:far-lfw} shows that targeting pairs with higher FAR allows the sampling estimator to match the full-baseline rate after approximately 900{,}000 pairs. The more limited gains compared to node classification are consistent with Table~\ref{tab:variance_reduction}: the auxiliary information available for LFW exhibits weak correlation with the target loss, yet informed sampling degrades gracefully and never underperforms its uninformed counterpart. When the contrastive loss itself is precomputed once over all pairs at the beginning of training and used directly as auxiliary information, substantially better results are obtained (Figure~\ref{fig:idealcaselfwlearning}, Appendix). This oracle setting provides an upper bound on achievable gains, confirming that the limited results on LFW stem from the weakness of the available auxiliary information rather than from any intrinsic limitation of the method.

\section{Conclusion \& perspectives}\label{sec:conclusion}
In this article, we present both theoretical and empirical arguments demonstrating, in a quantifiable manner, that when estimating empirical pairwise loss from a limited subset of a massive database, it is preferable to sample pairs of observations directly rather than selecting observations and forming pairs a posteriori. This trade-off between computational efficiency and estimation accuracy is further improved when auxiliary information allows the definition of pairwise sampling weights positively correlated with the loss. As we showed, minimizing such an empirical pairwise loss estimate based on observations that are no longer i.i.d. yields predictive functions with generalization guarantees, while numerical experiments confirm that task-relevant auxiliary information leads to performance close to full-pair learning at a significantly reduced computational cost. Although sharper learning-rate bounds accounting for variance reduction remain to be established, this work paves the way for more frugal pairwise learning methods through appropriate data selection schemes.

\bibliographystyle{plainnat}
\bibliography{references}

@article{CHAUVET2026,
title = {Exponential inequalities for sampling designs},
author = {G. Chauvet and M. Gerber},
journal = {Statistics \& Probability Letters},
volume = {232},
pages = {110654},
year = {2026},
issn = {0167-7152},
doi = {https://doi.org/10.1016/j.spl.2026.110654},
url = {https://www.sciencedirect.com/science/article/pii/S0167715226000180}
}

@article{yates53,
  title={Selection without replacement from within strata with probability proportional to size},
  author={Yates, F. and Grundy, P.M.},
  journal={Journal of the Royal Statistical Society: Series B (Methodological)},
  volume={15},
  number={2},
  pages={253--261},
  year={1953},
  publisher={Wiley Online Library}
}

@inproceedings{musgrave2020,
author = {Musgrave, K. and Belongie, S. and Lim, S.N.},
title = {A Metric Learning Reality Check},
year = {2020},
isbn = {978-3-030-58594-5},
publisher = {Springer-Verlag},
address = {Berlin, Heidelberg},
url = {https://doi.org/10.1007/978-3-030-58595-2_41},
doi = {10.1007/978-3-030-58595-2_41},
booktitle = {Computer Vision – ECCV 2020: 16th European Conference, Glasgow, UK, August 23–28, 2020, Proceedings, Part XXV},
pages = {681–699},
numpages = {19},
location = {Glasgow, United Kingdom}
}

@article{BLRG12,
author={Boistard, H. and Lopuha\^a, H.P. and Ruiz-Gazen, A.},
title={Approximation of rejective sampling inclusion probabilities and application to high order correlations},
journal={Electron. J. Statist.},
volume={6},
pages={1967-1983},
year={2012},
}

@InProceedings{pmlr-v63-clemencon64,
  title = 	 {{Learning from Survey Training Samples: Rate Bounds for Horvitz-Thompson Risk Minimizers}},
  author = 	 {Cl\'emen\c{c}on, S. and Bertail, P. and Papa, G.},
  booktitle = 	 {Proceedings of The 8th Asian Conference on Machine Learning},
  pages = 	 {142--157},
  year = 	 {2016},
  editor = 	 {Durrant, Robert J. and Kim, Kee-Eung},
  volume = 	 {63},
  series = 	 {Proceedings of Machine Learning Research},
  publisher =    {PMLR}
}

@article{joag1983negative,
  title={Negative association of random variables with applications},
  author={Joag-Dev, K. and Proschan, F.},
  journal={The Annals of Statistics},
  pages={286--295},
  year={1983},
  publisher={JSTOR}
}

@article{CLEM2015a,
  author={Cl\'emen\c{c}on, S. and Bertail, P. and Chautru, E. and Papa, G.},
 title = {Optimal survey schemes for stochastic gradient descent with applications to {M}-estimation},
	DOI= "10.1051/ps/2018021",
	journal = {ESAIM: PS},
	year = 2019,
	volume = 23,
	pages = "310-337",
}

@article{Dupacova,
  title={A note on rejective sampling},
  author={Dupacova, J.},
  journal={Contribution to Statistics (J. Hajek memorial volume) Academia Prague},
  year={1979},
  pages={71-78}
}

@article{Hoeffding63,
 author={Hoeffding, W.},
 year={1963},
title={Probability inequalities for sums of bounded
random variables},
 journal={J. Amer. Statist. Assoc.}, 
 volume={58},
 pages={13-30}
}

@BOOK{tille2006sampling,
  title={{Sampling algorithms}},
  author={Till{\'e}, Y.},
  year={2006},
  publisher={Springer Series in Statistics}
}

@ARTICLE{HT51, 
	author={Horvitz, D.G. and Thompson, D.J.}, 
	year={1951}, 
	title={{A generalization of sampling without replacement from a finite universe}}, 
	journal={JASA}, 
	volume={47}, 
	pages={663-685}
}

@ARTICLE{Hajek64,
	author={Hajek, J.}, 
	year={1964}, 
	title={{Asymptotic theory of rejective sampling with varying probabilities from a finite population}}, 
	journal={The Annals of Mathematical Statistics},
	volume={35}, 
	number={4}, 
	pages={1491-1523}
}

@ARTICLE{HR62, 
	author={Hartley, H.O. and Rao, J.N.K.}, 
	year={1962}, 
	title={{Sampling with unequal probabilities and without replacement}}, 
	journal={Ann. Math. Statist.},
	volume={33}, 
	pages={350-374}
}

@ARTICLE{CBC2017,
	author = {Cl\'{e}men\c{c}on, S. and Bertail, P. and Chautru, E.},
	title = {{Sampling and Empirical Risk Minimization}},
	journal = {Statistics},
    volume = {51},
number = {1},
pages = {30--42},
year = {2017},
publisher = {Taylor \& Francis},
doi = {10.1080/02331888.2016.1259810},


URL = { 
    
        https://doi.org/10.1080/02331888.2016.1259810
    
    

},
eprint = { 
    
        https://doi.org/10.1080/02331888.2016.1259810
    
    

}
}

@article{rao88,
 ISSN = {01621459, 1537274X},
 URL = {http://www.jstor.org/stable/2288945},
 author = {J. N. K. Rao and C. F. J. Wu},
 journal = {Journal of the American Statistical Association},
 number = {401},
 pages = {231--241},
 publisher = {[American Statistical Association, Taylor \& Francis, Ltd.]},
 title = {Resampling Inference With Complex Survey Data},
 urldate = {2025-12-03},
 volume = {83},
 year = {1988}
}

@article{sampford1967,
   author  = {Sampford, M. R.},
   title   = {On sampling without replacement with unequal probabilities
              of selection},
   journal = {Biometrika},
   volume  = {54},
   pages   = {499--513},
   year    = {1967}
}

@article{branden2012,
author = {Br{\"{a}}nd\'en, P. and Jonasson, J.},
title = {Negative Dependence in Sampling},
journal = {Scandinavian Journal of Statistics},
volume = {39},
number = {4},
pages = {830-838},
doi = {https://doi.org/10.1111/j.1467-9469.2011.00766.x},
url = {https://onlinelibrary.wiley.com/doi/abs/10.1111/j.1467-9469.2011.00766.x},
eprint = {https://onlinelibrary.wiley.com/doi/pdf/10.1111/j.1467-9469.2011.00766.x},
year = {2012}
}

@inproceedings{gcn,
  title     = {Semi-Supervised Classification with Graph Convolutional Networks},
  author    = {Kipf, T.N. and Welling, M.},
  booktitle = {International Conference on Learning Representations},
  year      = {2017}
}

@inproceedings{pyg,
  title     = {Fast Graph Representation Learning with {PyTorch Geometric}},
  author    = {Fey, M. and Lenssen, J.E.},
  booktitle = {ICLR Workshop on Representation Learning on Graphs and Manifolds},
  year      = {2019}
}

@article{movielens,
  author  = {Harper, F.M. and Konstan, J.A.},
  title   = {The {MovieLens} Datasets: History and Context},
  journal = {ACM Transactions on Interactive Intelligent Systems},
  year    = {2015}
}

@inproceedings{yang2016revisiting,
  title     = {Revisiting Semi-Supervised Learning with Graph Embeddings},
  author    = {Yang, Z. and Cohen, W. and Salakhutdinov, R.},
  booktitle = {International Conference on Machine Learning},
  year      = {2016}
}

@BOOK{Lee90,
	author={Lee, A.J.},
	year={1990},
	title={{${U}$-statistics: Theory and practice}},
	publisher={Marcel Dekker, Inc.},
	address={New York}
}

@article{BardenetMaillard,
  title={Concentration inequalities for sampling without replacement},
  author={Bardenet, R and Maillard, O.A.},
  journal={Bernoulli},
  year={2015},
  volume={21},
  number={3},
  pages={1361-1385}
}

@article{CCB16,
  author  = {S. Cl{{\'e}}men{\c{c}}on and I. Colin and A. Bellet},
  title   = {Scaling-up Empirical Risk Minimization: Optimization of Incomplete $U$-statistics},
  journal = {Journal of Machine Learning Research},
  year    = {2016},
  volume  = {17},
  number  = {76},
  pages   = {1--36},
  url     = {http://jmlr.org/papers/v17/15-012.html}
}

@book{BHS15,
  title={Metric Learning},
  author={Bellet, A. and Habrard, A. and Sebban, M.},
  isbn={9781627053662},
  series={Synthesis Lectures on Artificial Intelligence and Machine Learning},
  url={https://books.google.fr/books?id=SvzRBgAAQBAJ},
  year={2015},
  publisher={Morgan \& Claypool Publishers}
}

@InProceedings{pmlr-v139-bertail21a,
  title = 	 {Learning from Biased Data: A Semi-Parametric Approach},
  author =       {Bertail, P. and Cl{\'e}men{\c{c}}on, S. and Guyonvarch, Y. and Noiry, N.},
  booktitle = 	 {Proceedings of the 38th International Conference on Machine Learning},
  pages = 	 {803--812},
  year = 	 {2021},
  editor = 	 {Meila, Marina and Zhang, Tong},
  volume = 	 {139},
  series = 	 {Proceedings of Machine Learning Research},
  month = 	 {18--24 Jul},
  publisher =    {PMLR},
  url = 	 {https://proceedings.mlr.press/v139/bertail21a.html}}

@INPROCEEDINGS{hadsell06,
  author={Hadsell, R. and Chopra, S. and LeCun, Y.},
  booktitle={2006 IEEE Computer Society Conference on Computer Vision and Pattern Recognition (CVPR'06)}, 
  title={Dimensionality Reduction by Learning an Invariant Mapping}, 
  year={2006},
  volume={2},
  number={},
  pages={1735-1742},
  doi={10.1109/CVPR.2006.100}}

@BOOK{PenaGine99,
  author        = {V. de la Pena and E. Gin\'e},
  title         = {{Decoupling: from Dependence to Independence}},
  publisher     = {Springer},
  year          = {1999},
}

@ARTICLE{CLV08,
  author =       {S. Cl\'emen\c{c}on and G. Lugosi and N. Vayatis},
  title =        {Ranking and empirical risk minimization of {U}-statistics},
  journal =      {The Annals of Statistics},
  year =         {2008},
  volume =       {36},
  number =       {2},
  pages =        {844-874},
}

@InProceedings{houdre03,
author="Houdr{\'e}, C.
and Reynaud-Bouret, P.",
editor="Gin{\'e}, E.
and Houdr{\'e}, C.
and Nualart, D.",
title="Exponential Inequalities, with Constants, for U-statistics of Order Two",
booktitle="Stochastic Inequalities and Applications",
year="2003",
publisher="Birkh{\"a}user Basel",
address="Basel",
pages="55--69",
isbn="978-3-0348-8069-5"
}

@article{haung2008lfw,
author = {Huang, G. and Mattar, M. and Berg, T. and Learned-Miller, E.},
year = {2008},
month = {10},
pages = {},
title = {Labeled Faces in the Wild: A Database for Studying Face Recognition in Unconstrained Environments},
journal = {Tech. rep.}
}

@book{knuth1973art,
  title={The art of computer programming},
  author={Knuth, D.E. and others},
  volume={3},
  year={1973},
  publisher={Addison-Wesley Reading, MA}
}

@article{chao1982general,
 ISSN = {00063444, 14643510},
 URL = {http://www.jstor.org/stable/2336002},
 author = {M. T. Chao},
 journal = {Biometrika},
 number = {3},
 pages = {653--656},
 publisher = {[Oxford University Press, Biometrika Trust]},
 title = {A General Purpose Unequal Probability Sampling Plan},
 urldate = {2026-03-16},
 volume = {69},
 year = {1982}
}

@article{vitter1985random,
  title={Random sampling with a reservoir},
  author={Vitter, J.S.},
  journal={ACM Transactions on Mathematical Software (TOMS)},
  volume={11},
  number={1},
  pages={37--57},
  year={1985},
  publisher={ACM New York, NY, USA}
}

@inproceedings{schroff2015facenet,
  title={Facenet: A unified embedding for face recognition and clustering},
  author={Schroff, F. and Kalenichenko, D. and Philbin, J.},
  booktitle={Proceedings of the IEEE conference on computer vision and pattern recognition},
  pages={815--823},
  year={2015}
}

@misc{zhou2025,
      title={Randomized Pairwise Learning with Adaptive Sampling: A PAC-Bayes Analysis}, 
      author={S. Zhou and Y. Lei and A. Kabán},
      year={2025},
      eprint={2504.02957},
      archivePrefix={arXiv},
      primaryClass={cs.LG},
      url={https://arxiv.org/abs/2504.02957}, 
}

@article{chen2020simclr,
  author       = {T. Chen and
                  S. Kornblith and
                  M. Norouzi and
                  G.E. Hinton},
  title        = {A Simple Framework for Contrastive Learning of Visual Representations},
  journal      = {CoRR},
  volume       = {abs/2002.05709},
  year         = {2020},
  url          = {https://arxiv.org/abs/2002.05709},
  eprinttype   = {arXiv},
  eprint       = {2002.05709},
  bibsource    = {dblp computer science bibliography, https://dblp.org}
}

@article{zbontar2021barlowtwins,
  author       = {J. Zbontar and
                  L. Jing and
                  I. Misra and
                  Y. LeCun and
                  S. Deny},
  title        = {Barlow Twins: Self-Supervised Learning via Redundancy Reduction},
  journal      = {CoRR},
  volume       = {abs/2103.03230},
  year         = {2021},
  url          = {https://arxiv.org/abs/2103.03230},
  eprinttype   = {arXiv},
  eprint       = {2103.03230},
  bibsource    = {dblp computer science bibliography, https://dblp.org}
}

@article{robinson2020contrastive,
  author       = {J. Robinson and
                  C.Y. Chuang and
                  S. Sra and
                  S. Jegelka},
  title        = {Contrastive Learning with Hard Negative Samples},
  journal      = {CoRR},
  volume       = {abs/2010.04592},
  year         = {2020},
  url          = {https://arxiv.org/abs/2010.04592},
  eprinttype   = {arXiv},
  eprint       = {2010.04592},
  bibsource    = {dblp computer science bibliography, https://dblp.org}
}

@inproceedings{katharopoulos2018not,
  title={Not all samples are created equal: Deep learning with importance sampling},
  author={Katharopoulos, A. and Fleuret, F.},
  booktitle={International Conference on Machine Learning},
  pages={2525--2534},
  year={2018},
  organization={PMLR}
}

@inproceedings{zhao2015stochastic,
  title={Stochastic optimization with importance sampling for regularized loss minimization},
  author={Zhao, P. and Zhang, T.},
  booktitle={international Conference on Machine Learning},
  pages={1--9},
  year={2015},
  organization={PMLR}
}

@misc{wu2018samplingm,
      title={Sampling Matters in Deep Embedding Learning}, 
      author={C.Y. Wu and R. Manmatha and A.J. Smola and P. Krähenbühl},
      year={2018},
      eprint={1706.07567},
      archivePrefix={arXiv},
      primaryClass={cs.CV},
      url={https://arxiv.org/abs/1706.07567}, 
}

@INPROCEEDINGS{kumar2009attribute,
  author={Kumar, N. and Berg, A.C. and Belhumeur, P.N. and Nayar, S.K.},
  booktitle={2009 IEEE 12th International Conference on Computer Vision}, 
  title={Attribute and simile classifiers for face verification}, 
  year={2009},
  volume={},
  number={},
  pages={365-372},
  doi={10.1109/ICCV.2009.5459250}}

@article{ohlsson1998sequential,
  title={Sequential poisson sampling},
  author={Ohlsson, E.},
  journal={Journal of Official Statistics},
  volume={14},
  number={2},
  pages={149},
  year={1998},
  publisher={Statistics Sweden (SCB)}
}

\newpage
\appendix

\section*{Organisation of the appendix}

The appendix is organised as follows. 
\begin{itemize}
    \item Appendix \ref{app:notations} gathers additional information regarding notations and survey schemes.
    \item Appendix \ref{app:proofs} presents the proofs of the paper.
    \item Appendix \ref{appendix:experiment_settings} presents the experiment settings.
    \item Appendix \ref{app:additional-results} presents additional results.
    \item Then you can find the paper checklist.
\end{itemize}
\section{Notations and sampling plans}\label{app:notations}

Table \ref{tab:reminder_notations} states the differences between the observation and pair level notation, depending on the sampling plan considered.

\begin{table}[H]\begin{center}
\caption{Notation reminder}
\label{tab:reminder_notations}
\begin{tabular}{@{}llcc@{}}
\toprule
 & & \textbf{Obs.-level} & \textbf{Pair-level} \\
\midrule
\multicolumn{2}{@{}l}{\textit{Generic}} \\
 & Population          & $I = \{1,\ldots,N\}$ & $J = \{(i,j): i<j\}$ \\
 & Sampling plan       & $D_N$                & $\bar{D}_{\bar{N}}$ \\
 & Incl.\ indicators   & $\epsilon_i$         & $\bar{\epsilon}_{i,j}$ \\
 & HT estimator        & $\widetilde{U}_{D_N}$ & $\bar{U}_{\bar{D}_{\bar{N}}}$ \\
\midrule
\multicolumn{2}{@{}l}{\textit{Poisson}} \\
 & Sampling plan & $P_N$             & $\bar{P}_{\bar{N}}$\\
 & 1st-order incl.\ prob. & $p_i$             & $\bar{p}_{i,j}$ \\
 & 2nd-order incl.\ prob. & $p_i p_j$         & $\bar{p}_{i,j}\,\bar{p}_{k,l}$ \\
\midrule
\multicolumn{2}{@{}l}{\textit{Rejective}} \\
 & Sampling plan & $R_N$             & $\bar{R}_{\bar{N}}$\\
 & 1st-order incl.\ prob. & $\pi_i$\textsuperscript{\dag} & $\bar{\pi}_{i,j}$ \\
 & 2nd-order incl.\ prob. & $\pi_{i,j}$       & $\bar{\pi}_{(i,j),(k,l)}$ \\
\bottomrule
\end{tabular}
\end{center}
\end{table}

\dag~The symbol $\pi_i$ also denotes the generic first-order inclusion probability of an arbitrary design. Under rejective sampling, the same symbol is used, as it is simply an instance of the generic case commonly used in the literature.

\subsection{Additional examples of sampling plans}\label{app:example-plans}

The main text focuses on four designs: Bernoulli and Poisson sampling (independent designs with variable sample size, equal and unequal inclusion probabilities respectively), and their fixed-size counterparts SRSWOR and Rejective sampling, which are negatively associated \citep{joag1983negative}. Table~\ref{tab:sampling_plans_full} summarises these designs together with three additional ones that are popular in practice. 

\begin{table}[t]
\centering
\footnotesize
\caption{Extended comparison of sampling designs. NA: negative association established; SR: strong Rayleigh (implies CNA$^+$, hence NA); CNA: conditionally negatively associated (implies NA).}
\label{tab:sampling_plans_full}
\setlength{\tabcolsep}{5pt}
\begin{tabular}{l cc cc c c}
\toprule
\textbf{Design}
  & \multicolumn{2}{c}{\textbf{Incl.\ prob.}}
  & \multicolumn{2}{c}{\textbf{Sample size}}
  & \textbf{NA}
  & \textbf{Online}\\
\cmidrule(lr){2-3}\cmidrule(lr){4-5}
& Equal & Unequal & Fixed & Variable & &  \\
\midrule
Bernoulli   & \checkmark &            &            & \checkmark
  & \textemdash\ (indep.)  & \ding{55}  \\
Poisson     &            & \checkmark &            & \checkmark
  & \textemdash\ (indep.)  & \ding{55}   \\
SRSWOR      & \checkmark &            & \checkmark &
  & \checkmark\ \citep{joag1983negative} & \ding{55}   \\
Rejective   &            & \checkmark & \checkmark &
  & \checkmark\ \citep{joag1983negative} & \ding{55}   \\
Reservoir   & \checkmark &            & \checkmark &
  & \checkmark\ \citep{joag1983negative} & \checkmark   \\
Chao        &            & \checkmark & \checkmark &
  & CNA \citep{CHAUVET2026}   & \checkmark   \\
Rao-Sampford &           & \checkmark & \checkmark &
  & SR \citep{branden2012}                    & \ding{55}  \\
\bottomrule
\end{tabular}
\end{table}

For each design, we report whether inclusion probabilities are equal or unequal, whether the sample size is fixed, whether negative association (NA) has been established in the literature, whether the design is \emph{online} (i.e.\ applicable to data streams where $N$ is not known in advance), and which results of the paper carry over. The three additional designs are described below.

\subsubsection*{Reservoir sampling}
Reservoir sampling, introduced in \cite{knuth1973art}, is a sequential design that draws a uniform sample of fixed size $n$ from a stream of observations whose total size $N$ need not be known in advance. Its first- and second-order inclusion probabilities coincide with those of SRSWOR:
\begin{equation}
  \pi_i = \frac{n}{N}, \qquad
  \pi_{ij} = \frac{n(n-1)}{N(N-1)}.
\end{equation}
Because the resulting distribution is identical to SRSWOR, reservoir sampling inherits the negative association property of SRSWOR \citep{joag1983negative}. A practical advantage over SRSWOR is its $O(N)$ memory footprint: only the current reservoir of size $n$ needs to be stored. We refer to \cite{knuth1973art} and \cite{vitter1985random} for efficient implementations.

\subsection{Chao sampling}
Chao sampling \citep{chao1982general} extends reservoir sampling to \emph{unequal} inclusion probabilities, making it suitable for probability-proportional-to-size ($\pi$ps) sampling on data streams. Given positive weights $W_i > 0$, the target first-order inclusion probabilities at time $k$ are defined by
\begin{equation}
  \sum_{i=1}^{k} \pi(k;i) = n, \qquad \pi(k;i) \propto W_i,
\end{equation}
yielding an $O(N)$ algorithm with no rejection step.~\citet{CHAUVET2026} establish that Chao's procedure satisfies the conditional negative association (CNA) property, which
implies NA. 

\subsubsection*{Rao-Sampford sampling}
Rao-Sampford sampling \citep{rao88,sampford1967} is a fixed-size $\pi$ps design with prescribed inclusion probabilities $\boldsymbol{\pi} = (\pi_1, \dots, \pi_N)$ satisfying $n\pi_i < 1$ for all $i$. Algorithm~\ref{alg:rao_sampford} gives its standard
rejective implementation; a non-rejective implementation exists \citep{tille2006sampling} but is not required here.
 
\begin{algorithm}[H]
\caption{Rao-Sampford sampling -- rejective implementation
         \citep{rao88,sampford1967}}
\label{alg:rao_sampford}
\begin{algorithmic}[1]
  \REQUIRE Population $\mathcal{I} = \{1,\dots,N\}$,
           probabilities $\boldsymbol{p} = (p_1,\dots,p_N)$
           with $\sum p_i = 1$ and $np_i < 1$ for all $i$,
           target size $n$
  \ENSURE  Sample $s \subseteq \mathcal{I}$ with $|s| = n$
  \REPEAT
    \STATE Draw the first unit $i_1$ with probabilities
           $(p_1, \dots, p_N)$ (with replacement)
    \STATE Draw $n-1$ additional units independently with replacement,
           each with probabilities proportional to
           $p_j/(1-p_j)$ for $j = 1, \dots, N$
  \UNTIL{all $n$ drawn units are distinct}
  \STATE \textbf{return} $s \leftarrow \{i_1, \dots, i_n\}$
\end{algorithmic}
\end{algorithm}
 
\citet{branden2012} show that Rao-Sampford sampling satisfies the strong Rayleigh property, which implies CNA$^+$ and in particular NA. Second-order inclusion probabilities do not admit a closed form but can be approximated to first order via the Hartley-Rao formula.

\subsection{Approximation of second-order inclusion probabilities}

Let us denote $\mathcal{I}$ the population of size $N$. Let $\boldsymbol{p} = \{p_1, \dots, p_N\}$ be the inclusion probabilities of the Poisson sampling design. We have, for all $i$, $0\leq p_i\leq 1$ and we have :

\begin{equation}
    \sum_{i=1}^N p_i = n
\end{equation}
where $n$ is the expected sampling size of the Poisson sampling. The probability of any sample $s$ is given by :

\begin{equation}
    P(s) = \prod_{i\in s} p_i \prod_{i \notin s} (1-p_i)
\end{equation}

The Poisson sampling yiels independent samples, which means that the inclusion probabilities are as follows, for $i, j \in \mathcal{I}$ and $k\geq 2$ :

\begin{align}
    p_i &= P(i \in s) = \sum_{s \ni i} P(s)  = p_i\\
    p_{ij} &= P(i \in s, j \in s)  = p_i \times p_j \\
    p_{i_1 \cdots i_k} &= P(i_1 \in s, \dots, i_k \in s) = \prod_{t=1}^k p_{i_t}
\end{align}

All of the above is also true for a uniform Poisson, in which case all $p_i$ are equal. 

The corresponding rejective sampling design, with $n$ fixed observations, is such that the probability of any sample $s$ is given by :

\begin{align}
    P_{RS}(s) =
    \begin{cases} 
        c  \prod_{i\in s} p_i \prod_{i \notin s} (1-p_i) \qquad \text{if the size of } s \text{ is } n,\\
        0 \qquad \text{otherwise.}
    \end{cases}
\end{align}
where $c$ is a normalizing constant ensuring that $P_{RS}$ sums to one over all samples of size $n$. Computing it  requires summing over all $\binom{N}{n}$ subsets of size $n$, which explains why the first-order inclusion probabilities $\pi_i$ of the rejective design do not admit a closed form. \citet{Hajek64} showed that they satisfy $\sum_{i=1}^N \pi_i = n$ and are well approximated by $p_i$ for large populations.  \citet{HR62} also proposed an estimation of second-order inclusion probabilities for large $N$ :

\begin{equation}
    \pi_{ij} = \frac{n-1}{n}\,p_i p_j
    \left[1 + \frac{p_i + p_j}{n}\right]
    + O(N^{-3})
\end{equation}

In a different asymptotic setting, \citet{BLRG12} found that, for $k \geq 2$ and $A_k = \{i_1, i_2, \dots, i_k\} \subset \{1,\dots,N\}$, the following approximations hold as $d \to \infty$, where $d = \sum_{k=1}^N p_k(1-p_k)$ : 

\begin{equation}
    \pi_{i_1, i_2, \dots, i_k} = \pi_{i_1} \pi_{i_2} \cdots \pi_{i_k} \times
    \left(1 - d^{-1} \sum_{i,j \in A_k:\, i < j} (1-p_i)(1-p_j) 
    + O(d^{-2})\right),
\end{equation}

\begin{equation}
    \pi_{i_1, i_2, \dots, i_k} = \pi_{i_1} \pi_{i_2} \cdots \pi_{i_k} \times
    \left(1 - d^{-1} \sum_{i,j \in A_k:\, i < j} (1-\pi_i)(1-\pi_j) 
    + O(d^{-2})\right),
\end{equation}
where $O(d^{-2})$ holds uniformly in $i_1, i_2, \dots, i_k$.



\section{Technical proofs}\label{app:proofs}

\subsection{Proof of Theorem \ref{thm:opt_pair}}

Let us investigate pairwise Poisson sampling strategies to estimate $U=\mathbb{E}[\ell(Z,Z')]$, avoiding to store in memory all the pairs $\{(Z_i,Z_j):\; 1\leq i<j\leq N\}$, in contrast to 
\begin{equation}
\widehat{U}_N=\frac{2}{N(N-1)}\sum_{i<j}\ell (Z_i,Z_j).
\end{equation}

Let $\bar{P}_{\bar{N}}$ be a Poisson plan on the population $\mathcal{J}=\{(i,j)\in \mathcal{I}^2:\; i<j\}$ based on the auxiliary information $\mathbf{W}_N$, leading to a sample of pairs $\bar{S}\subset \mathcal{P}(\mathcal{J})$ and the related estimator
\begin{equation}\label{eq:HT_pair}
\bar{U}_{\bar{P}_{\bar{N}}}=\frac{2}{N(N-1)}\sum_{(i,j)\in \bar{S}}\frac{\ell(Z_i,Z_j)}{\bar{\pi}_{(i,j)}}=\frac{2}{N(N-1)}\sum_{i<j}\frac{\bar{\epsilon}_{(i,j)}}{\bar{\pi}_{(i,j)}}\ell(Z_i,Z_j),
\end{equation}
denoting by $\bar{\pi}_{(i,j)}(\mathbf{W}_N)=\mathbb{P}((i,j)\in \bar{S}\mid \mathbf{W}_N)$, $i<j$, its inclusion probabilities. Since it is an unbiased estimator of $\widehat{U}$, namely
\begin{equation}\label{eq:HT_pair_exp}
\mathbb{E}[\bar{U}_{\bar{P}_{\bar{N}}}\mid \mathbf{W}_N, \mathbf{Z}_N ]=\widehat{U} \text{ almost-surely},
\end{equation}
we have
\begin{equation}\label{eq:full_var_pair}
Var(\bar{U}_{\bar{P}_{\bar{N}}})=\mathbb{E}[Var(\bar{U}_{\bar{P}_{\bar{N}}}\mid \mathbf{W}_N, \mathbf{Z}_N )]+Var(\widehat{U}_N).
\end{equation}
The second term on the right hand side of the equation above being independent from $\bar{P}_{\bar{N}}$, the smaller the conditional variance , the more accurate the estimator \eqref{eq:HT_pair} in the least squares sense. The conditional variance is expressed as 
\begin{equation}\label{eq:var-cond-poisonns}
Var(\bar{U}_{\bar{P}_{\bar{N}}}\mid \mathbf{W}_N, \mathbf{Z}_N )=\frac{1}{\bar{N}^2}\sum_{i<j}\left(\frac{1}{\bar{p}_{i,j}(\mathbf{W}_N)}-1\right)\ell^2(Z_i,Z_j),
\end{equation}
so that its conditional expectation given the auxiliary information available $\mathbf{W}_N$ is
\begin{equation}
\label{eq:cond_exp__given_wn_poisson_pairwise}
\mathbb{E}[Var(\bar{U}_{\bar{P}_{\bar{N}}}\mid \mathbf{W}_N, \mathbf{Z}_N )\mid \mathbf{W}_N]=\frac{1}{\bar{N}^2}\sum_{i<j}\left(\frac{1}{\bar{p}_{i,j}(\mathbf{W}_N)}-1\right)\mathbb{E}[\ell^2(Z_i,Z_j)\mid W_i, W_j],
\end{equation}
which can be seen, by means of a Lagrange multipliers argument, as minimum under the constraint that the expected size is $\bar{n}$ (\textit{i.e.} $\bar{n}=\sum_{i<j}\bar{p}_{i,j}(\mathbf{W}_N)$ with probability $1$) for the inclusion probabilities corresponding to the link function
$\rho^*(W,W')=\sqrt{\mathbb{E}[\ell^2(Z,Z')\mid W,W']}$, namely 
\begin{equation}
\bar{p}^*_{i,j}(\mathbf{W}_N)=\bar{n}\frac{\sqrt{\mathbb{E}[\ell^2(Z_i,Z_j)\mid W_i,W_j]}}{\sum_{k<l} \sqrt{\mathbb{E}[\ell^2(Z_k,Z_l)\mid W_k,W_l]}}
\end{equation}
in the case where $\bar{p}^*_{i,j}(\mathbf{W}_N)\leq 1$ for all $(i,j)\in \mathcal{J}$ with probability $1$, which condition is automatically fulfilled when Assumption \ref{hyp1} holds true. Denoting by $\bar{P}_{\bar{N}}^*$ the Poisson sampling plan defined by these optimal weights, the minimum conditional variance is given by
\begin{equation*}
\mathbb{E}[Var(\bar{U}_{\bar{P}^*_{\bar{N}}}\mid \mathbf{W}_N, \mathbf{Z}_N )\mid \mathbf{W}_N]=\frac{1}{\bar{N}^2}\sum_{i<j}\left(\frac{1}{\bar{p}^*_{i,j}(\mathbf{W}_N)}-1\right)\mathbb{E}[\ell(Z_i,Z_j)^2\mid W_i,W_j].
\end{equation*}
By integrating over $\mathbf{W}_N$, we obtain that the minimum value for the expected conditional variance (the first term on the right hand side of \eqref{eq:full_var_pair}) is given by
\begin{multline}
\mathbb{E}[Var(\bar{U}_{\bar{P}^*_{\bar{N}}}\mid \mathbf{W}_N, \mathbf{Z}_N )\mid \mathbf{W}_N]\\=\frac{4}{N^2(N-1)^2}\sum_{i<j}\left(\frac{\sum_{k<l}\rho^*(W_k,W_l)}{\bar{n}\rho^*(W_i,W_j)}-1\right)\mathbb{E}[\ell(Z_i,Z_j)^2\mid W_i,W_j]\\=
\frac{4}{N^2(N-1)^2}\sum_{i<j}\left(\frac{\sum_{k<l}\rho^*(W_k,W_l)}{\bar{n}\rho^*(W_i,W_j)}-1\right)\left( \rho^*(W_i,W_j) \right)^2\\= \frac{4}{\bar{n}N^2(N-1)^2}\sum_{i<j}\rho^*(W_i,W_j)\sum_{k<l}\rho^*(W_k,W_l)  -\frac{4}{N^2(N-1)^2}\sum_{i<j} \left(\bar{\rho}^*(W_i,W_j) \right)^2,
\end{multline}
and its expectation w.r.t. $\mathbf{W}_N$ is
\begin{equation}
\bar{\sigma}^2_*= \left(1-\frac{2}{N(N-1)}\right)\frac{1}{\bar{n}}\left(\mathbb{E}[\rho^*(W,W')]\right)^2+ \left(\frac{1}{\bar{n}}-1\right)\frac{2}{N(N-1)}\mathbb{E}\left[(\rho^*(W, W'))^2\right].
\end{equation}

\subsection{Proof of Proposition \ref{prop:gain}}

The strategy that consists in using a sampling plan $D_N$ on the population $\mathcal{I}$ based on the auxiliary information $\mathbf{W}_N$, just like in the pointwise situation, yields the estimator:
\begin{equation}
\widetilde{U}_{D_N}=\frac{2}{N(N-1)}\sum_{i<j\; \text{ in } S}\frac{\ell(Z_i,Z_j)}{\pi_{i,j}}=\frac{2}{N(N-1)}\sum_{1\leq i<j\leq N}\frac{\epsilon_i\epsilon_j}{\pi_{i,j}}\ell(Z_i,Z_j),
\end{equation}
denoting by $\pi_{i,j}(\mathbf{W}_N)$, $i<j$, its second order inclusion probabilities.

In the case of a Poisson plan $P_N$, as the second order inclusion probabilities are expressed as a function of the first order inclusion probabilities, the conditional variance is given by
\begin{multline}
\frac{N^2(N-1)^2}{4}Var\left( \widetilde{U}_{P_N}  \mid \mathbf{W}_N, \mathbf{Z}_N\right)=\sum_{i<j}\frac{\ell^2(Z_i,Z_j)}{p_i^2p_j^2}Var(\epsilon_i\epsilon_j\mid \mathbf{W}_N, \mathbf{Z}_N)\\+\sum_{i<j,\; k<l,\; (i,j)\neq(k,l)}\frac{\ell(Z_i,Z_j)\ell(Z_k,Z_l)}{p_ip_jp_kp_l}Cov(\epsilon_i\epsilon_j,\epsilon_k\epsilon_l)\\=
\sum_{i<j}\frac{1-p_ip_j}{p_ip_j}\ell^2(Z_i,Z_j)+\sum_{i<j,\; i<l,\; j\neq l}\frac{1-p_i}{p_i}\ell(Z_i,Z_j)\ell(Z_i,Z_l)\\+
\sum_{k<i<j}\frac{1-p_i}{p_i}\ell(Z_i,Z_j)\ell(Z_k,Z_i)
+\sum_{i<j<l}\frac{1-p_j}{p_j}\ell(Z_i,Z_j)\ell(Z_j,Z_l)\\
+\sum_{i<j,\; k<j,\; i\neq k}\frac{1-p_j}{p_j}\ell(Z_i,Z_j)\ell(Z_k,Z_j)\\
= \sum_{i<j}\left(\frac{1}{p_ip_j}-1\right)\ell^2(Z_i,Z_j)+2\sum_{i,\; j<l,\; j\neq i,\; l\neq i}\left(\frac{1}{p_i}-1\right)\ell(Z_i,Z_j)\ell(Z_i,Z_l).
\end{multline}
We deduce that
\begin{multline}
\mathbb{E}\left[Var\left( \widetilde{U}_{P_N}  \mid \mathbf{W}_N, \mathbf{Z}_N\right) \mid \mathbf{W}_N \right]=\frac{1}{\bar{N}^2}\sum_{i<j}\left(\frac{1}{p_ip_j}-1\right)\mathbb{E}\left[\ell^2(Z_i,Z_j)\mid W_i,W_j\right]\\+\frac{2}{\bar{N}^2}\sum_{i,\; j<l,\; j\neq i,\; l\neq i}\left(\frac{1}{p_i}-1\right)\mathbb{E}\left[\ell(Z_i,Z_j)\ell(Z_i,Z_l)\mid W_i, W_j, W_l\right].
\end{multline}
Observing that the first term on the right hand side of the equation above has exactly the same form as \eqref{eq:cond_exp__given_wn_poisson_pairwise} (with $\bar{p}_{i,j}=p_ip_j$ for $i<j$), we almost-surely have, under the constraint $\sum_{i<j}p_ip_j=\bar{n}$,
\begin{equation}\label{eq:bound1}
\frac{1}{\bar{N}^2}\sum_{i<j}\left(\frac{1}{p_ip_j}-1\right)\mathbb{E}\left[\ell^2(Z_i,Z_j)\mid W_i,W_j\right]\geq \mathbb{E}[Var(\bar{U}_{\bar{P}^*_{\bar{N}}}\mid \mathbf{W}_N, \mathbf{Z}_N )\mid \mathbf{W}_N]
\end{equation}
by virtue of Theorem \ref{thm:opt_pair}.
This implies that, with probability one,
\begin{multline}
\mathbb{E}\left[Var\left( \widetilde{U}_{P_N}  \mid \mathbf{W}_N, \mathbf{Z}_N\right) \mid \mathbf{W}_N \right]\geq \mathbb{E}[Var(\bar{U}_{\bar{P}^*_{\bar{N}}}\mid \mathbf{W}_N, \mathbf{Z}_N )\mid \mathbf{W}_N]
\\+\frac{2}{\bar{N}^2}\sum_{i,\; j<l,\; j\neq i,\; l\neq i}\left(\frac{1}{p_i}-1\right)\mathbb{E}\left[\ell(Z_i,Z_j)\ell(Z_i,Z_l)\mid W_i, W_j, W_l\right].
\end{multline}
In addition, it is straightforward to see that the bound \eqref{eq:bound1} is an equality with probability one if and only if, for all $i<j$,
$$
p_ip_j=\mathbb{E}\left[ \ell^2(Z_i,Z_j)\mid W_i,W_j \right]= \psi(W_i)\psi(W_j) \text{ almost-surely}.
$$
Hence, if Assumption \ref{hyp2} is fulfilled, we necessarily have:
\begin{equation}
\mathbb{E}\left[Var\left( \widetilde{U}_{P_N}  \mid \mathbf{W}_N, \mathbf{Z}_N\right) \right]> \mathbb{E}[Var(\bar{U}_{\bar{P}^*_{\bar{N}}}\mid \mathbf{W}_N, \mathbf{Z}_N )]=\bar{\sigma}_*^2.
\end{equation}

\subsection{Extension to (pairwise) rejective sampling}
\label{appendix:rejective_sampling_ext}

As we will now show, to a certain extent, theoretical results similar to those proved above can be obtained for certain survey sampling plans with fixed size, the conditional Poisson (rejective) scheme in particular. 

Consider a survey scheme $\bar{D}_{\bar{N}}$ in the population $\mathcal{J}:=\{(i,j):\; 1\leq i<j\leq N\}$ based on information $\mathbf{W}_N$, yielding a sample of pairs $\bar{S}$ and the estimator of \eqref{eq:emp_pair_risk}:
\begin{equation*}
\bar{U}_{\bar{D}_{\bar{N}}}=\frac{2}{N(N-1)}\sum_{(i,j) \in \bar{S}}\frac{\ell(Z_i,Z_j)}{\bar{\pi}_{i,j}}
=\frac{2}{N(N-1)}\sum_{i<j}\frac{\bar{\epsilon}_{i,j}}{\bar{\pi}_{i,j}}\ell(Z_i,Z_j),
\end{equation*}
with $\bar{\pi}_{i,j}(\mathbf{W}_N)=\mathbb{P}(\bar{\epsilon}_{i,j}=1\mid \mathbf{Z}_N,\mathbf{W}_N)$ and $\bar{\epsilon}_{i,j}=\mathbb{I}\{(i,j)\in \bar{S}\}$ for $i<j$. If it is of fixed size $\bar{n}\in (0,\bar{N})$, \textit{i.e.} if $\sum_{i<j}\bar{\epsilon}_{i,j}=\bar{n}$ with probability $1$, its conditional variance involves the second order inclusion probabilities $$\bar{\pi}_{(i,j),(k,l)}(\mathbf{W}_N)=\mathbb{P}(\bar{\epsilon}_{i,j}\bar{\epsilon}_{k,l}=1=1\mid \mathbf{Z}_N,\mathbf{W}_N)$$ with $i<j$, $k<l$ and $(i,j)\neq (k,l)$, and is given by the Sen-Yates-Grundy formula:
\begin{multline}
Var(\bar{U}_{\bar{D}_{\bar{N}}}\mid \mathbf{W}_N, \mathbf{Z}_N )=\\
\frac{1}{2\bar{N}^2}\sum_{\substack{i<j, k<l\\ (i,j)\neq (k,l)}}\left(\frac{\ell(Z_i,Z_j)}{\bar{\pi}_{i,j}(\mathbf{W}_N)}-\frac{\ell(Z_k,Z_l)}{\bar{\pi}_{k,l}(\mathbf{W}_N)}\right)^2\left(\bar{\pi}_{i,j}(\mathbf{W}_N)\bar{\pi}_{k,l}(\mathbf{W}_N)-\bar{\pi}_{(i,j),(k,l)}(\mathbf{W}_N)\right),
\end{multline}
whose conditional expectation given $\mathbf{W}_N$ is

\begin{multline}
\mathbb{E}[\operatorname{Var}(\bar{U}_{\bar{D}_{\bar{N}}}\mid \mathbf{W}_N, \mathbf{Z}_N )\mid \mathbf{W}_N]
= \frac{1}{2\bar{N}^2}\sum_{\substack{i<j,\, k<l}}\\
\mathbb{E}\!\left[\left(\frac{\ell(Z_i,Z_j)}{\bar{\pi}_{i,j}(\mathbf{W}_N)}
  -\frac{\ell(Z_k,Z_l)}{\bar{\pi}_{k,l}(\mathbf{W}_N)}\right)^{\!2}
  \,\middle|\, \mathbf{W}_N\right]\\
\times\left(\bar{\pi}_{i,j}(\mathbf{W}_N)\bar{\pi}_{k,l}(\mathbf{W}_N)
  -\bar{\pi}_{(i,j),(k,l)}(\mathbf{W}_N)\right).
\end{multline}

For all $(i,j),\; (k,l)$ s.t. $i<j$ and $k<l$, we have
\begin{multline}\label{eq1}
\mathbb{E}\left[\left(\frac{\ell(Z_i,Z_j)}{\bar{\pi}_{i,j}(\mathbf{W}_N)}-\frac{\ell(Z_k,Z_l)}{\bar{\pi}_{k,l}(\mathbf{W}_N)}\right)^2\mid \mathbf{W}_N\right]=\frac{\mathbb{E}[\ell^2(Z_i,Z_j)\mid W_i,W_j]}{\bar{\pi}^2_{i,j}(\mathbf{W}_N)}\\+\frac{\mathbb{E}[\ell^2(Z_k,Z_l)\mid W_k,W_l]}{\bar{\pi}^2_{k,l}(\mathbf{W}_N)} \\
-2\frac{\mathbb{E}[\ell(Z_i,Z_j)\ell(Z_k,Z_l)\mid W_i,W_j,W_k,W_l]}{\bar{\pi}_{i,j}(\mathbf{W}_N)\times \bar{\pi}_{k,l}(\mathbf{W}_N)}.
\end{multline}
One way to deal with the difficulties arising from the presence of second-order inclusion probabilities is to use approximations in terms of first order inclusion probabilities in a specific asymptotic framework. In \cite{HR62}, it is proved for rejective sampling, successive sampling and for randomized systematic sampling, that for $\bar{n}$ fixed, we almost-surely have, uniformly in $(i,j)\neq (k,l)$:
\begin{equation}
\bar{\pi}_{i,j}(\mathbf{W}_N)\bar{\pi}_{k,l}(\mathbf{W}_N)-\bar{\pi}_{(i,j),(k,l)}(\mathbf{W}_N)\sim \frac{\bar{\pi}_{i,j}(\mathbf{W}_N)\bar{\pi}_{k,l}(\mathbf{W}_N)}{\bar{n}} \text{ as } N\to\infty.
\end{equation}
More precisely, it follows from Eq. (5.15) in \cite{HR62} that the expected conditional variance can be asymptotically approximated as follows:
 \begin{multline}
\mathbb{E}[Var(\bar{U}_{\bar{D}_{\bar{N}}}\mid \mathbf{W}_N, \mathbf{Z}_N )\mid \mathbf{W}_N]\\ \sim
\frac{1}{2\bar{N}^2}\sum_{i<j, k<l}\mathbb{E}\left[\left(\frac{\ell(Z_i,Z_j)}{\bar{\pi}_{i,j}(\mathbf{W}_N)}-\frac{\ell(Z_k,Z_l)}{\bar{\pi}_{k,l}(\mathbf{W}_N)}\right)^2\mid \mathbf{W}_N\right]\frac{\bar{\pi}_{i,j}(\mathbf{W}_N)\bar{\pi}_{k,l}(\mathbf{W}_N)}{\bar{n}},
 \end{multline}
as $N$ tends to infinity. Using \eqref{eq1}, it is then easy to see that minimizing the above asymptotic equivalent of the expected variance (setting its gradient w.r.t the $\bar{\pi}_{i,j}$'s to $0$) leads to the solution: $\forall i<j$,
$$
\bar{\pi}_{i,j}(\mathbf{W}_N)=\bar{p}^*_{i,j}(\mathbf{W}_N).
$$

\subsection{Concentration Bounds under Poisson Sampling}\label{app:sec4-poisson}

We derive a nonasymptotic bound for the additional deviation term $\bar U_{\bar{D}_{\bar{N}}}(\theta)-\hat{U}_N(\theta)$ induced by sampling pairs. Recall that $\mathcal{J}=\{(i,j):1\leq i<j\leq N\}$ denotes the population of all pairs with cardinality $\bar{N}=N(N-1)/2$, and that under a sampling design $\bar{D}_{\bar{N}}$ on $\mathcal{J}$ with first-order inclusion probabilities $(\bar{\pi}_{i,j})_{(i,j)\in\mathcal{J}}$, the HT estimator of the complete U-statistic reads
$$\bar{U}_{\bar{D}_{\bar{N}}}(\theta)=\frac{1}{\bar{N}}\sum_{(i,j)\in \mathcal{J}}
\frac{\bar{\epsilon}_{i,j}}{\bar{\pi}_{i,j}}\ell_{\theta}(Z_i,Z_j), 
\qquad \ell_\theta(z, z'):= \ell(\theta, (z,z')).$$
We also recall the complete U-statistic
$$\hat{U}_N(\theta)=\frac{1}{\bar{N}}\sum_{(i,j)\in \mathcal{J}}\ell_\theta(Z_i,Z_j).$$

For any $\theta \in \Theta$, we define the centered contributions 
$$X_{ij}(\theta):=\left(\frac{\bar\epsilon_{i,j}}{\bar \pi_{i,j}}-1\right)\ell_\theta(Z_i, Z_j), \qquad (i,j)\in \mathcal{J},$$ 
so that 
\begin{equation}\label{eq:survey_diff}
    \bar{U}_{\bar{D}_{\bar{N}}}(\theta)-\hat{U}_N(\theta)=\frac{1}{\bar{N}}\sum_{(i,j)\in \mathcal{J}}X_{ij}(\theta).
\end{equation}

Conditionally on $\mathcal{A}_{N}$ and on the observed data $(Z_1, \cdots Z_N)$, we have $\mathbb{E}[X_{ij}|\mathcal{A}_N,Z_{1:N}]=0$ since $\mathbb{E}[\bar\epsilon_{i,j}|\mathcal{A}_N]=\bar\pi_{i,j}$.

\paragraph{Assumptions}
Throughout this subsection, we consider Poisson sampling, i.e., $(\bar{\epsilon}_{i,j})_{(i,j)\in J}$ are conditionally independent given $\mathcal{A}_N := \sigma(W_1, \ldots, W_N)$, the sigma-algebra generated by the auxiliary information, with $\bar{\epsilon}_{i,j} \sim \mathrm{Bernoulli}(\bar{\pi}_{i,j})$. In addition, we assume:
\begin{enumerate}
    \item[(A1)] (\textbf{Bounded loss}) $|\ell_\theta(z,z')| \leq M$ for all 
    $\theta\in\Theta$ and all $(z,z')\in\mathcal{Z}^2$.
    \item[(A2)] (\textbf{Non-degenerate inclusion probabilities}) 
    $\bar\pi_{i,j}\ge \bar{\pi}_{\min}>0$ for all $(i,j)\in\mathcal{J}$.
\end{enumerate}

Note that (A1) allows for signed losses. This relaxation is harmless: the same Bernstein argument applies with $B := 2M/\bar{\pi}_{\min}$ instead of $M/\bar{\pi}_{\min}$, yielding identical bounds up to a constant factor. The non-negativity assumption is never used in Section~\ref{sec:estimation}; it is only required here for the concentration bounds, and only in the bounded form $|\ell_\theta| \leq M$.

\begin{lemma}[Conditional Bernstein bound]\label{lem:bernstein_poisson_pairs}
    Under (A1)--(A2) and Poisson sampling on pairs, for any fixed $\theta \in \Theta$ 
    and any $t>0$,
    $$\mathbb{P}\left(\left|\sum_{(i,j)\in\mathcal{J}} X_{ij}(\theta)\right| \geq t 
    \bigg|\mathcal{A}_N, Z_{1:N}\right)\leq 2\exp\left(-\frac{t^2}{2V(\theta)+\frac{2}{3}Bt}\right),$$
    where $B:=\frac{2M}{\bar\pi_{\min}}$ and $V(\theta):=\sum_{(i,j)\in\mathcal{J}}
    \ell_\theta(Z_i,Z_j)^2\frac{1-\bar\pi_{i,j}}{\bar\pi_{i,j}}$. 
    Consequently,
    $$\mathbb{P}\left(\left|\bar U_{\bar{D}_{\bar{N}}}(\theta)-\hat U_N(\theta)\right| 
    \geq \frac{t}{\bar{N}} \bigg|\mathcal{A}_N, Z_{1:N}\right)\leq 
    2\exp\left(-\frac{t^2}{2V(\theta)+\frac{2}{3}Bt}\right).$$
\end{lemma}

\begin{proof}
    We fix $\theta\in\Theta$ and work conditionally on $\mathcal{A}_N, Z_{1:N}$. Under Poisson sampling on pairs, the random variables $(X_{ij}(\theta))_{(i,j)\in\mathcal{J}}$ are independent and centered. First, by (A1)--(A2), 
    $$|X_{ij}(\theta)|=\left|\left(\frac{\bar\epsilon_{i,j}}{\bar \pi_{i,j}}-1\right)
    \ell_\theta(Z_i,Z_j)\right|\leq \left(\frac{1}{\bar\pi_{i,j}}+1\right)M 
    \leq \frac{2M}{\bar\pi_{\min}}=:B,$$
    where we used $|\ell_\theta| \leq M$ and $\bar\pi_{i,j} \leq 1$. Second, using $\mathrm{Var}(\bar\epsilon_{i,j})=\bar\pi_{i,j}(1-\bar\pi_{i,j})$ and the fact that $X_{ij}(\theta)$ is a scalar multiple of $\bar\epsilon_{i,j}-\bar\pi_{i,j}$, we obtain 
    $$\mathrm{Var}\left(X_{ij}(\theta)|\mathcal{A}_N, Z_{1:N}\right)=
    \ell_\theta(Z_i, Z_j)^2\frac{1-\bar\pi_{i,j}}{\bar\pi_{i,j}}.$$
    Summing over $(i,j)\in \mathcal{J}$ yields the conditional variance proxy $V(\theta)$. The inequality follows from Bernstein's inequality for sums of independent, centered random variables, applied conditionally. The final statement follows from Equation~\eqref{eq:survey_diff}.
\end{proof}

Next, we turn Lemma~\ref{lem:bernstein_poisson_pairs} into a uniform deviation bound over $\Theta$. We assume that there exists $L>0$ and a norm $\|\cdot\|$ on $\Theta \subset\mathbb{R}^p$ such that for all $\theta, \theta' \in \Theta$ and all $(z,z')\in\mathcal{Z}^2$:
\begin{equation}\label{def:lipchitz}
|\ell_\theta(z,z')-\ell_{\theta'}(z,z')|\leq L||\theta-\theta'||. 
\end{equation}

\begin{theorem}[Uniform control of the survey deviation (Poisson)]
\label{th:uniform-poisson-pair}
Assume (A1)--(A2) and Equation~\eqref{def:lipchitz} and consider Poisson sampling on pairs. Let $\Theta_\eta$ be an $\eta$-net of $\Theta$ for $\|\cdot\|$, with cardinality $\mathcal{N}(\Theta,\|\cdot\|,\eta)$. Then, for any $\eta>0$ and any $\delta\in(0,1)$, with conditional probability at least $1-\delta$ given $(\mathcal{A}_N, Z_{1:N})$,
$$\sup_{\theta\in\Theta}\left|\bar U_{\bar{D}_{\bar{N}}}(\theta)-\hat{U}_N(\theta)\right|
\leq \max_{\theta'\in\Theta_\eta}\left|\bar U_{\bar{D}_{\bar{N}}}(\theta')-
\hat{U}_N(\theta')\right|+\frac{L}{\bar\pi_{\min}}\eta.$$
Moreover, 
$$\max_{\theta'\in\Theta_\eta}\left|\bar U_{\bar{D}_{\bar{N}}}(\theta')-\hat{U}_N(\theta')\right|
\leq\frac{1}{\bar{N}}\left[\sqrt{2V_\eta\log\left(\frac{2\mathcal{N}(\Theta,\|\cdot\|,\eta)}{\delta}\right)}
+\frac{2}{3}B\log\left(\frac{2\mathcal{N}(\Theta,\|\cdot\|,\eta)}{\delta}\right)\right],$$
where $B=\frac{2M}{\bar\pi_{\min}}$ and $V_\eta:=\max_{\theta'\in\Theta_\eta}V(\theta')$.
\end{theorem}

\begin{proof}
For any $\theta\in\Theta$, let $\theta^{\sharp}\in \Theta_\eta$ satisfy $\|\theta-\theta^{\sharp}\|\leq \eta$. Using Equation~\eqref{def:lipchitz} and the bound $\left|\frac{\bar\epsilon_{ij}}{\bar\pi_{ij}}-1\right|\le 1/\bar\pi_{\min}+1 \leq 2/\bar\pi_{\min}$, we have, for any $(i,j)\in\mathcal{J}$,
$$\left|X_{ij}(\theta)-X_{ij}(\theta^{\sharp})\right|\leq \frac{2L}{\bar\pi_{\min}}\eta.$$
Summing over $(i,j)\in\mathcal{J}$ and using $|\mathcal{J}|=\bar N$ yields 
$$\left|\sum_{(i,j)\in\mathcal{J}}X_{ij}(\theta)-\sum_{(i,j)\in\mathcal{J}}
X_{ij}(\theta^{\sharp})\right|\leq \bar{N}\cdot\frac{2L}{\bar\pi_{\min}}\eta.$$
Multiplying by $1/\bar{N}$ and using Equation~\eqref{eq:survey_diff}, we obtain 
$$\left|\bar U_{\bar{D}_{\bar{N}}}(\theta)-\hat{U}_N(\theta)\right|\leq 
\left|\bar U_{\bar{D}_{\bar{N}}}(\theta^\sharp)-\hat{U}_N(\theta^\sharp)\right|
+\frac{2L}{\bar{\pi}_{\min}}\eta.$$
Taking the supremum over $\theta\in\Theta$ gives the first inequality. We then apply Lemma~\ref{lem:bernstein_poisson_pairs} to each $\theta'\in\Theta_{\eta}$ and use the union bound to conclude.
\end{proof}

\begin{remark}[Variance proxy under a general link function]\label{rem:Vrho}
Under a pairwise Poisson plan with link function $\rho$ satisfying Assumption~\ref{hyp1}, the inclusion probabilities are $\bar\pi_{i,j} = \bar{n}\rho(W_i,W_j)/S^\rho_N$ where 
$S^\rho_N = \sum_{k<l}\rho(W_k,W_l)$. Bounding $\ell^2_\theta \leq M^2$, the variance 
proxy satisfies
\begin{equation}\label{eq:Vrho_bound}
    V(\theta) \leq M^2\sum_{i<j}\frac{1-\bar\pi_{i,j}}{\bar\pi_{i,j}} 
    \leq M^2\sum_{i<j}\left(\frac{S^\rho_N}{\bar{n}\,\rho(W_i,W_j)}-1\right) =: V^\rho.
\end{equation}
Substituting $V_\eta \leq V^\rho$ into Theorem~\ref{th:uniform-poisson-pair} yields the variance-aware bound stated in Theorem~\ref{thm:main-learning} of the main text. The bound is tightest when $\rho$ is aligned with the loss: the optimal link function $\rho^*$ of Theorem~\ref{thm:opt_pair} simultaneously minimizes estimation variance in Section~\ref{sec:estimation} and minimizes $V^\rho$ here. In the special case 
$\rho \equiv 1$ (uniform sampling), $V^\rho = M^2\bar{N}(\bar{N}/\bar{n}-1)$, recovering a generic $O(\bar{N}/\bar{n})$ constant.
\end{remark}

\subsection{Concentration Bounds under Rejective Sampling}\label{app:sec4-rejective}
Now, we extend Theorem~\ref{th:uniform-poisson-pair} to rejective (fixed-size) sampling 
designs. We consider a rejective sample design $\bar{D}_{\bar{N}}$ that selects exactly $\bar{n}$ pairs from $\mathcal{J}$ according to some inclusion probabilities $(\bar\pi_{i,j})_{(i,j)\in\mathcal{J}}$ satisfying $\sum_{(i,j)}\bar\pi_{i,j}=\bar{n}$. Under such designs, the inclusion indicators 
$(\bar\epsilon_{i,j})_{(i,j)\in\mathcal{J}}$ are no longer independent but satisfy negative association.

\begin{remark}[Recall on negative association]
We say that random variables $X_1,\cdots, X_m$ are negatively associated~\citep{joag1983negative} if for any two disjoint subsets $I,J \subseteq\{1,\cdots,m\}$ and any decreasing functions $f:\mathbb{R}^{|I|}\to \mathbb{R}$ and $g:\mathbb{R}^{|J|}\to \mathbb{R}$,
$$\mathbb{E}[f(X_I)\,g(X_J)]\le \mathbb{E}[f(X_I)]\mathbb{E}[g(X_J)],$$ 
where $X_I = (X_i)_{i\in I}$. A key property is that sampling without replacement induces negative association among the inclusion indicators~\citep{joag1983negative}. 
\end{remark} 

\begin{lemma}[Bernstein for negative association~\citep{BardenetMaillard}]
\label{lem:bernstein-rejective}
Let $X_1,\ldots,X_m$ be negatively associated, centered random variables with $|X_k| \le B$ almost surely. Then for any $t > 0$,
$$\mathbb{P}\left(\left|\sum_{k=1}^m X_k\right| \ge t\right)
\le 2\exp\left(-\frac{t^2}{2\sum_{k=1}^m \mathrm{Var}(X_k) + \frac{2}{3}Bt}\right).$$
\end{lemma}

\begin{theorem}[Uniform control of the survey deviation (rejective)]
\label{th:uniform-rejective-pair}
Under the same assumptions as Theorem~\ref{th:uniform-poisson-pair}, but with rejective sampling instead of Poisson sampling, the conclusion of Theorem~\ref{th:uniform-poisson-pair} continues to hold: for any $\eta>0$ and any $\delta\in(0,1)$, with probability at least $1-\delta$,
$$\sup_{\theta\in\Theta}\left|\bar U_{\bar{D}_{\bar{N}}}(\theta)-\hat{U}_N(\theta)\right|
\leq \max_{\theta'\in\Theta_\eta}\left|\bar U_{\bar{D}_{\bar{N}}}(\theta')-
\hat{U}_N(\theta')\right|+\frac{2L}{\bar\pi_{\min}}\eta,$$
and 
$$\max_{\theta'\in\Theta_\eta}\left|\bar U_{\bar{D}_{\bar{N}}}(\theta')-\hat{U}_N(\theta')\right|
\leq\frac{1}{\bar{N}}\left[\sqrt{2V_\eta\log\left(\frac{2\mathcal{N}(\Theta,\|\cdot\|,\eta)}{\delta}\right)}
+\frac{2}{3}B\log\left(\frac{2\mathcal{N}(\Theta,\|\cdot\|,\eta)}{\delta}\right)\right].$$
\end{theorem}

\begin{proof}
The proof follows exactly the same structure as that of Theorem~\ref{th:uniform-poisson-pair}, with the difference that instead of applying Bernstein's inequality for independent variables 
(Lemma~\ref{lem:bernstein_poisson_pairs}), we apply Lemma~\ref{lem:bernstein-rejective} for negatively associated random variables. Since the $X_{ij}(\theta)$ are negatively associated under rejective sampling (as non-decreasing functions of disjoint subsets of the negatively associated indicators 
$\bar\epsilon_{ij}$), the same bound holds. The union bound over the $\eta$-net proceeds identically.
\end{proof}

\subsection{Proof of Corollary~\ref{cor:sample}}

We prove the sample complexity result under Poisson sampling; the rejective case follows identically by Theorem~\ref{th:uniform-rejective-pair}.

Let $\eta = 1/\sqrt{N}$ and set $\bar{n} = cN$ for some fixed constant $c > 0$. Substituting $V^\rho$ from Remark~\ref{rem:Vrho} into Theorem~\ref{th:uniform-poisson-pair}, and noting that under $\bar{n} = cN$,
$$V^\rho \leq M^2\sum_{i<j}\left(\frac{S^\rho_N}{\bar{n}\,\rho(W_i,W_j)}-1\right) 
\leq M^2\cdot\frac{\bar{N}^2}{\bar{n}}\cdot\frac{\bar\rho}{\rho_{\min}} 
= M^2\cdot\frac{\bar{N}^2}{cN}\cdot\frac{\bar\rho}{\rho_{\min}},$$
where $\bar\rho = S^\rho_N/\bar{N}$ and $\rho_{\min} = \min_{i<j}\rho(W_i,W_j)$, we obtain that the survey-induced term satisfies
$$\frac{1}{\bar{N}}\sqrt{2V^\rho\log\frac{2\mathcal{N}(\Theta)}{\delta}} 
= O\left(\sqrt{\frac{\log\mathcal{N}(\Theta)}{\bar{n}}}\right) 
= O\left(\sqrt{\frac{\log\mathcal{N}(\Theta)}{N}}\right).$$
The Lipschitz remainder satisfies $\frac{L}{\bar\pi_{\min}}\eta = O(1/\sqrt{N})$ by choice of $\eta$. Combined with the standard U-statistic bound $\|\widehat{U}_N - U\|_\infty = O_\mathbb{P}(1/\sqrt{N})$ from~\citep{CLV08}, the decomposition~\eqref{eq:decomp} gives
$$U(\bar{\theta}_{\bar{D}_{\bar{N}}}) - \inf_{\theta\in\Theta}U(\theta) \leq 
\left(\frac{C_1}{\sqrt{c}} + C_2\right)\sqrt{\frac{\log\mathcal{N}(\Theta)}{N}},$$
with high probability, where $C_1$ depends on $M$, $\bar\pi_{\min}$, $\bar\rho/\rho_{\min}$, and $C_2$ is the constant from~\citep{CLV08}. The computational reduction follows from $\bar{n}/\bar{N} = cN / (N(N-1)/2) \sim 2c/N \to 0$, so that the ratio of pairs used to pairs available vanishes as $N \to \infty$ for any fixed $c$. \qed

\section{Experiment settings}
\label{appendix:experiment_settings}

This section reports the setup and protocol for the experiments presented in Sections~\ref{sec:estimation} and~\ref{sec:learning}. The code to reproduce experiments is available on this anonymous GitHub \url{https://anonymous.4open.science/r/pairwise_estimation_and_learning-2503/}.

\subsection{Datasets and preprocessing}

\paragraph{Toy dataset.}
We construct a synthetic dataset of $N=1000$ observations drawn from a sparse mixture: with probability $p=0.10$, $x_i \sim \mathcal{N}(5, 0.3^2)$ (signal), otherwise $x_i \sim \mathcal{N}(0, 0.3^2)$ (noise). The pairwise loss is $\ell(x_i, x_j) = (x_i \cdot x_j)^2$, concentrated on signal-signal pairs ($0.9\%$ of pairs, carrying $93\%$ of $U_{\rm true}$, ${\rm CV}^2(\ell)=99.8$). 
The auxiliary information $\rho(x_i, x_j) = |x_i \cdot x_j|$ achieves ${\rm Corr}(\rho,\ell) = 0.97$.

\paragraph{Cora.}
Cora~\citep{yang2016revisiting}, as provided by the PyTorch Geometric library~\citep{pyg}, is a citation graph of $N=2708$ scientific articles (7 classes) with bag-of-words node features ($d=1433$). For the estimation experiments (Section~3), node features are projected onto $64$ dimensions via PCA and $\ell_2$-normalized, yielding embeddings $z_i \in \mathbb{R}^{64}$. The pairwise loss is a hinge loss over these embeddings:
\begin{equation}
\ell(z_i, z_j) = \begin{cases}
\max(0,\, d_{ij} - M_+) & \text{if } y_i = y_j \\
\max(0,\, M_- - d_{ij}) & \text{if } y_i \neq y_j
\end{cases}
\end{equation}
where $d_{ij} = 1 - \langle z_i, z_j\rangle$ is the cosine distance and $M_+ = 0.2$, $M_- = 0.8$. This yields $14.4\%$ active pairs.
The auxiliary information combines two sources available without computing $\ell$: the same hinge loss evaluated on PCA-8 projections of the node features, and the number of common neighbors in the citation graph,$\rho(W_i, W_j) = \ell_8(z_i^{(8)}, z_j^{(8)}) + 
|\mathcal{N}(i) \cap \mathcal{N}(j)|$, achieving $\mathrm{Corr}(\rho, \ell) = 0.36$.

For the learning experiments (Section~4), we use a two-layer GCN~\citep{gcn} with hidden dimension $128$ and embedding dimension $64$, trained with the same hinge loss over training node pairs. Node classification accuracy is evaluated using a $60/40$ train/test split. Hyperparameters were fixed a priori and not tuned, they are reported in the \textit{Learning} paragraph below.

\paragraph{MovieLens.}
MovieLens-100k~\citep{movielens} contains $100{,}000$ ratings from $943$ users on $1682$ items (films). Observations are items $Z_i$, with auxiliary information $W_i = (\mathrm{pop}_i, \bar{r}_i)$ where $\mathrm{pop}_i$ is the number of ratings and $\bar{r}_i$ is the mean rating of item $i$. The pairwise loss is a thresholded BPR-style preference asymmetry:
\begin{equation}
\ell(i, j) = \max\!\left(0,\; 
\frac{1}{|U_{ij}|}\sum_{u \in U_{ij}} 
\mathbf{1}\{r_{ui} > r_{uj}\} - \tau\right),
\end{equation}
where $U_{ij}$ denotes the set of users who rated both items  $i$ and $j$, and $\tau = 0.6$ is a threshold ensuring sparsity. 
Pairs with fewer than $10$ co-raters are assigned $\ell = 0$. This yields $4.6\%$ active pairs.

\paragraph{LFW.}
Labeled Faces in the Wild~\cite{haung2008lfw} is a face recognition dataset containing $N = 13{,}233$ face images of $5{,}749$ distinct identities. Images are embedded by a pretrained FaceNet model~\citep{schroff2015facenet} into $512$-dimensional vectors $z_i \in \mathbb{R}^{512}$, which are $\ell_2$-normalised before use.

For the estimation experiments, the pairwise loss is a False Acceptance Rate (FAR) indicator:
\begin{equation}
  \ell(i,j) = \mathbf{1}\!\left\{d_{ij} < \theta \;\text{ and }\;
  y_i \neq y_j\right\},
\end{equation}
where $d_{ij} = \|z_i - z_j\|_2$ and $\theta$ is the 10th-percentile quantile of pairwise Euclidean distances. Two auxiliary configurations are evaluated, both derived from the $73$ per-image face attributes available in the LFW attribute  file~\citep{kumar2009attribute}.
In the \emph{realistic} configuration, the auxiliary score is obs-additive: $\rho(i,j) = s_i + s_j$, where $s_i$ is the normalised value of the single attribute most correlated (Pearson) with the per-observation FAR marginal count $f_i = |\{j : d_{ij} < \theta,\, y_i \neq y_j\}|$. In the \emph{idealistic} configuration, $\rho$ is the exact pairwise FAR indicator $\rho(i,j) = \mathbf{1}[d_{ij} < \theta,\, y_i \neq y_j] + \varepsilon$, which requires full knowledge of identity labels on all pairs and serves as an oracle upper bound.

For the learning experiments, the auxiliary information uses the top-10 attributes most correlated with $f_i$ (same construction as above with $K=10$). We additionally consider a \emph{loss-based} proxy that requires no hand-crafted attributes. It evaluates the Siamese
contrastive loss~\citep{hadsell06} directly on the pretrained embeddings (before the learned projection):
\begin{equation}
\rho(i,j) = \tfrac{1}{2}\Bigl[
  \mathbf{1}[y_i=y_j]\cdot d_{ij}^2
  +\mathbf{1}[y_i\neq y_j]\cdot\max(0,\,1-d_{ij})^2
\Bigr],
\end{equation}
assigning high weight to hard positives (same identity, far apart) and hard negatives (different identities, close together).

\paragraph{Hard sampling baseline}

As a competitive alternate strategy, we also choose to implement a method we call Hard sampling, which is inspired by Hard negative mining but adapted to a supervised setting where positive pairs also need to be sampled. At each epoch, the current model embeddings are computed (with no gradient) and the per-pair loss is evaluated on this pool.
The $\bar{n}$ pairs with the highest loss are then selected and used for the gradient step. Since this method does not have tractable inclusion probabilities, it yields a biased estimator
of the population loss; it serves purely as a practical baseline.

\paragraph{Survey sampling implementation}
Since NumPy does not provide implementations of the survey sampling schemes considered in this work, we implemented all schemes from scratch. Inclusion probabilities were constructed under the following constraints: they must lie in $(0,1]$ and sum to the expected sample size, whether sampling individuals or pairs. For the baselines (Bernoulli and SRSWOR), equal inclusion probabilities were used. For Poisson and rejective sampling, inclusion probabilities were either proportional to the target (an ideal but unrealistic setting), or correlated with the absolute value of the target variable through a chosen link function (sigmoid or identity). Each scheme was implemented as follows.
\begin{itemize}
  \item \textbf{SRSWR} draws $n$ indices independently and uniformly: $O(n)$ time and memory per replicate, no special casing required.
  \item \textbf{Bernoulli and Poisson} each draw unit $i$ independently with probability $p_i$, requiring a single pass over the population: $O(N)$ time and memory per replicate. Realisations yielding an empty sample were discarded. To avoid allocating a $B \times N$ boolean matrix when $N$ is large, replicates are processed in chunks of size $k = \min(B, \lfloor c/N \rfloor)$ capped at $c = 200\,\text{MB}$, keeping peak memory at $O(k \cdot N)$.
  \item \textbf{SRSWOR} was implemented via hash-set rejection sampling, which draws indices uniformly at random and discards duplicates. When $n \ll N$ the expected number of draws is $O(n)$; when $n/N$ is non-negligible the method degrades to a partial Fisher-Yates shuffle at $O(N)$.
  \item \textbf{Rejective sampling} repeatedly draws a Poisson candidate and rejects it if the realised sample size differs from $n$. The acceptance probability is $p_{\mathrm{acc}} = \Pr\!\bigl(\sum_i \mathbf{1}[U_i < p_i] = n\bigr)$ with $U_i \overset{\mathrm{iid}}{\sim} \mathcal{U}(0,1)$, giving $O(N/p_{\mathrm{acc}})$ expected time per replicate. The vector $(p_i)_{i=1}^N$ is allocated once and shared across replicates; only the $B \times n$ index matrix is stored, yielding $O(N + B \cdot n)$ total memory. For larger-scale settings, we additionally implemented Ohlsson's sequential algorithm~\cite{ohlsson1998sequential}, which avoids repeated rejection and scales more gracefully when $p_{\mathrm{acc}}$ is small while still yielding the exact same conditional Poisson distribution, preserving the unbiasedness of the HT estimator.
\end{itemize}
Second-order inclusion probabilities $\pi_{ij}$ are required for Horvitz--Thompson variance estimation in the observation sampling case, while pair sampling relies on first-order inclusion probabilities of pairs directly. For independent schemes (Bernoulli, Poisson), $\pi_{ij} = p_i p_j$ is exact and $O(1)$ per pair. For SRSWOR and SRSWR, closed-form scalar expressions depending only on $N$ and $n$ apply. For rejective sampling, no closed form exists; we use the \citet{Hajek64} approximation $\pi_{ij} \approx p_i p_j\bigl(1 - (1-p_i)(1-p_j)/d\bigr)$ with $d = \sum_k p_k(1-p_k)$, following \citet{BLRG12}, which requires a single $O(N)$ precomputation pass after which each pair query is $O(1)$. Table~\ref{tab:sampling_complexity} summarises all costs.
\begin{table}[t]
\centering
\caption{Construction complexity and second-order inclusion probability structure for each sampling design ($N$: population size, $n$: expected sample size, $B$: number of replicates). Time complexity is identical for naive and vectorised implementations; only memory differs.}
\label{tab:sampling_complexity}
\begin{threeparttable}
\begin{tabular}{@{}llllll@{}}
\toprule
 & \multicolumn{3}{c}{Construction complexity} & \multicolumn{2}{c}{$\pi_{ij}$} \\
\cmidrule(lr){2-4}\cmidrule(lr){5-6}
Scheme & Time / rep & Mem.\ / rep & Mem.\ ($B$ reps) & Exact? & Cost \\
\midrule
SRSWR     & $O(n)$                      & $O(n)$    & $O(B{\cdot}n)$        & \checkmark    & $O(1)$ (scalar) \\
Bernoulli & $O(N)$                      & $O(N)$    & $O(k \cdot N)$\tnote{a} & \checkmark  & $O(1)$ / pair \\
Poisson   & $O(N)$                      & $O(N)$    & $O(k \cdot N)$\tnote{a} & \checkmark  & $O(1)$ / pair \\
SRSWOR    & $O(n)$\tnote{b}             & $O(n)$\tnote{b} & $O(B{\cdot}n)$   & \checkmark  & $O(1)$ (scalar) \\
Rejective & $O(N/p_{\mathrm{acc}})$     & $O(N)$    & $O(N + B{\cdot}n)$    & $\approx$\tnote{c} & $O(N)$ then $O(1)$ \\
Ohlsson\tnote{d}  & $O(N)$  & $O(N)$  & $O(N + B{\cdot}n)$  & $\approx$\tnote{c} & $O(N)$ then $O(1)$ \\
\bottomrule
\end{tabular}
\begin{tablenotes}
  \small
  \item[a] $k = \min\!\left(B,\lfloor c/N \rfloor\right)$, $c = 200\,\text{MB}$.
  \item[b] Expected; degrades to $O(N)$ via partial Fisher-Yates when $n/N$ is large.
  \item[c] Hájek (1964); $d = \sum_k p_k(1-p_k)$, precomputed in $O(N)$.
  \item[d] Alternative to rejective sampling with identical conditional Poisson distribution.
\end{tablenotes}
\end{threeparttable}
\end{table}

\paragraph{Estimation.}
Variance and MAE are estimated over $B = 1000$ independent replicates of the sampling procedure, using a fixed random seed for reproducibility. Confidence intervals on MAE are computed as $\pm 1.96 \times \hat{\sigma}/\sqrt{B}$, where $\hat{\sigma}$ is the empirical standard deviation over the $B$ replicates.

\paragraph{Learning hyperparameters.}
Both datasets (Cora and LFW) share the same pair-budget grid $\bar{n} \in \{2000, 5000, 10000, 20000, 50000\}$ and a maximum of 1000 epochs.  At each epoch, a pool of 200k candidate pairs is drawn at random from the training nodes/observations and subsampled to a budget of $\bar{n}$ pairs according to the selected scheme. The optimiser is Adam. Hyperparameters are tuned per dataset: learning rate $5\times10^{-3}$ (LFW) and $1\times10^{-3}$ (Cora); early stopping monitors test accuracy every 20 epochs with patience. Results are averaged over 5 seeds, a smaller number of replicates than for the estimation experiments, due to the computational cost of training. For each global seed $s$, the train/val/test split is fixed \emph{before} any computation of $\rho$, so all auxiliary-information configurations are evaluated on the exact same split. LFW uses a random partition of 8000/2000/3000 images; Cora a random 60/20/20 partition of its 2708 nodes (approximately 1624/542/542). Model initialisation and per-epoch pair sampling use a deterministic strategy-specific sub-seed $s_{\mathrm{strat}} = 31s + \sum_c \mathrm{ord}(c)$, where the sum runs over the ASCII codes of the strategy name and 31 is a small prime used to spread values, ensuring that each (global seed, strategy) pair maps to a distinct reproducible sub-seed.

\paragraph{Computational resources}
Estimation experiments required CPU resources only, as they are primarily memory-bound. All of them were obtained by using a single Intel Xeon Gold 6154.
\begin{itemize}
    \item Figure~\ref{fig:toy_example_results} was obtained in approximately 43 minutes and can readily be reproduced by a normal computer using the provided script.
    \item Table~\ref{tab:variance_reduction} completed in under 10 minutes, as it involves a single sample size.
    \item The experiments in Figure~\ref{fig:supp-cora-movielens} completed in under one hour for Cora and under 30 minutes for MovieLens.
    \item Figure~\ref{fig:supp-lfw} was more challenging. Storing the 87 million pairwise inclusion probabilities required chunking to avoid memory overflow, as described in the survey sampling implementation paragraph above. Smaller sampling sizes completed in under 10 minutes across all replicates, while larger sampling sizes required over 40 minutes due to this chunking overhead. In total, the 12 subfigures took approximately 4 hours to complete, even with precomputed cache. This is normal due to the high number of replicates.
\end{itemize}
Learning experiments required GPU acceleration.
\begin{itemize}
    \item Cora experiments in figures~\ref{fig:cora-learning} and \ref{fig:cora-learning-rej} were obtained using a single NVIDIA RTX 3090 (24\,GB). The experiments took at most 1.5 hours in early iterations, decreasing to approximately 15 minutes in later iterations as intermediate computations were cached. We acknowledge that this makes it difficult to report a single reliable runtime estimate for the full experiment.
    \item LFW experiments in figures~\ref{fig:cora-learning}, \ref{fig:cora-learning-rej} and~\ref{fig:idealcaselfwlearning} were the most computationally demanding and were run on a single NVIDIA L40S (46\,GB). Each auxiliary information configuration required approximately 2 hours across all sampling sizes and replicates. As with the Cora experiments, precomputed embeddings and auxiliary information reduced wall-clock time considerably; the true cost of a full cold run would likely be substantially higher.
\end{itemize}



\section{Additional Results}\label{app:additional-results}
This Section reports additional results obtained both to illustrate the interest of our sampling design for estimation and from a learning perspective. 

\paragraph{Pairwise loss estimation}

Figures~\ref{fig:supp-cora-movielens} and \ref{fig:supp-lfw} display mean absolute error (MAE) and variance as a function of $\bar{n}$ for Cora, MovieLens and LFW. Across all sampling budgets, informed pair sampling consistently dominates uniform pair sampling, which itself dominates observation-based sampling by a wide margin. This is consistent with results obtained in Table \ref{tab:variance_reduction}.

\begin{figure}[H]
    \centering
    \begin{subfigure}[t]{0.24\linewidth}
        \includegraphics[width=\linewidth]{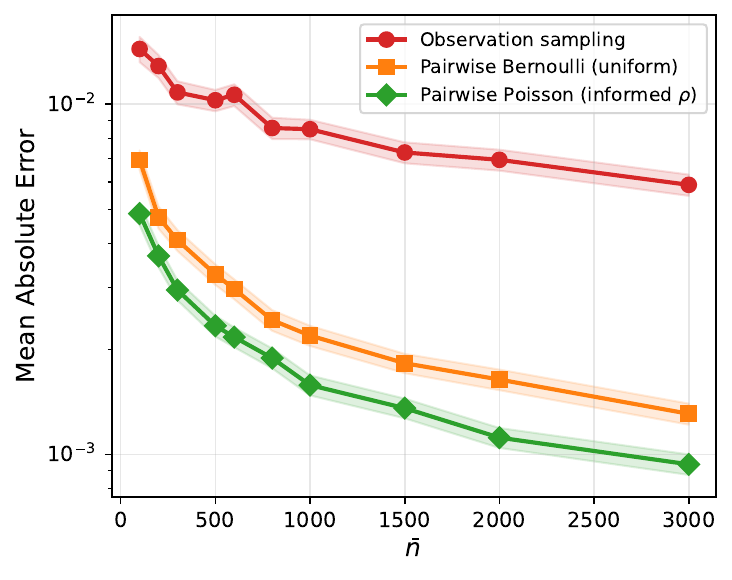}
    \end{subfigure}\begin{subfigure}[t]{0.24\linewidth}
        \includegraphics[width=\linewidth]{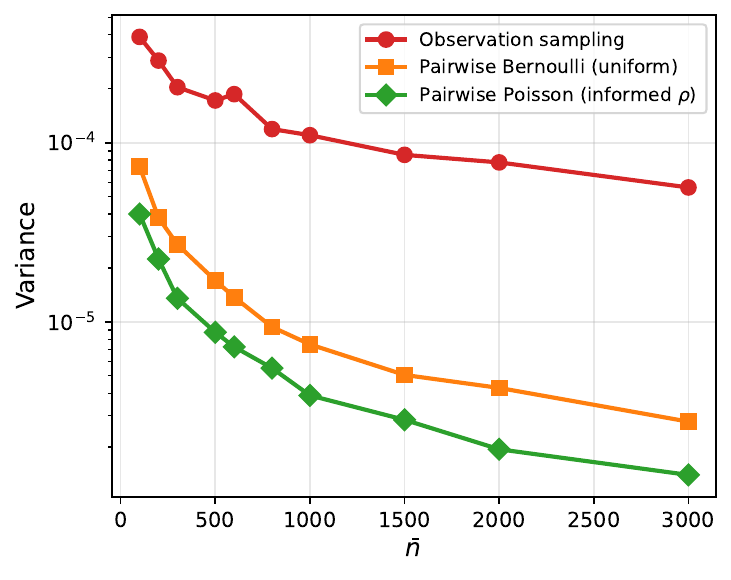}
        \caption{}
    \end{subfigure}\begin{subfigure}[t]{0.24\linewidth}
        \includegraphics[width=\linewidth]{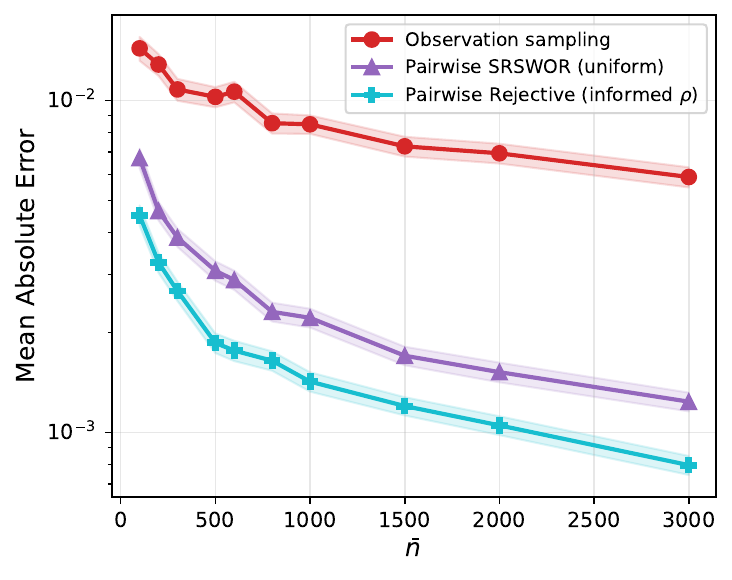}
        \caption{}
    \end{subfigure}\begin{subfigure}[t]{0.24\linewidth}
        \includegraphics[width=\linewidth]{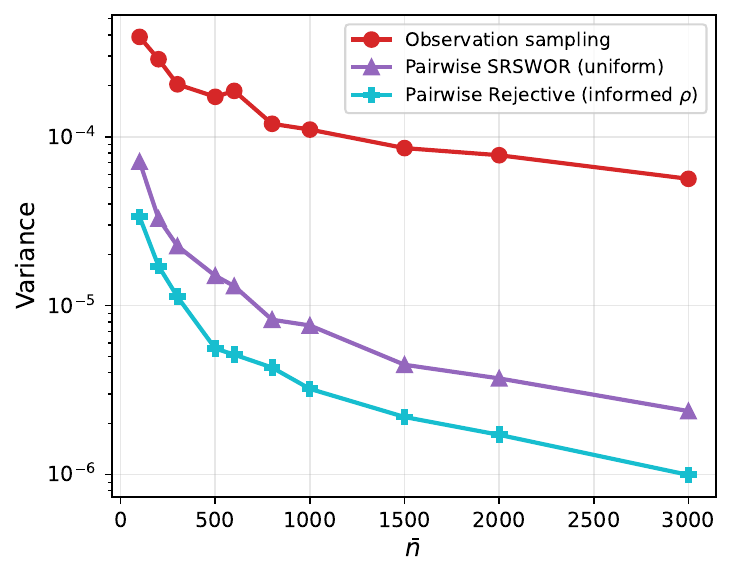}
        \caption{}
    \end{subfigure}\\
    \begin{subfigure}[t]{0.24\linewidth}
        \includegraphics[width=\linewidth]{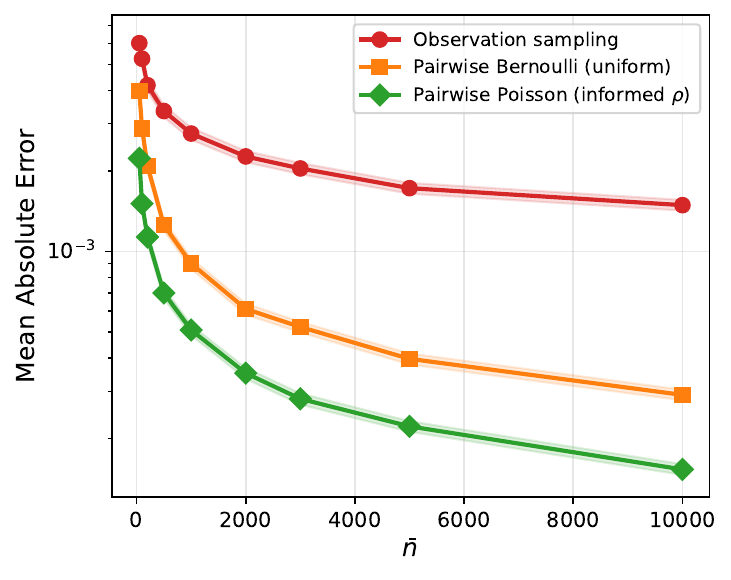}
        \caption{}
    \end{subfigure}\begin{subfigure}[t]{0.24\linewidth}
        \includegraphics[width=\linewidth]{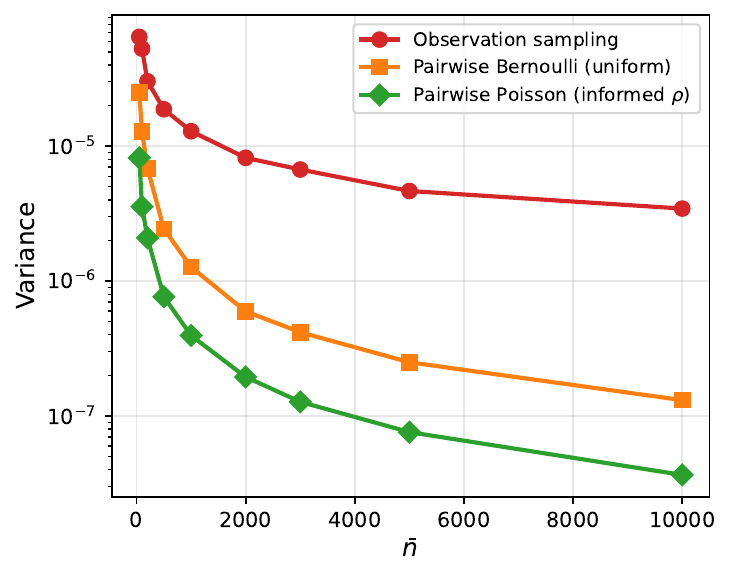}
        \caption{}
    \end{subfigure}\begin{subfigure}[t]{0.24\linewidth}
        \includegraphics[width=\linewidth]{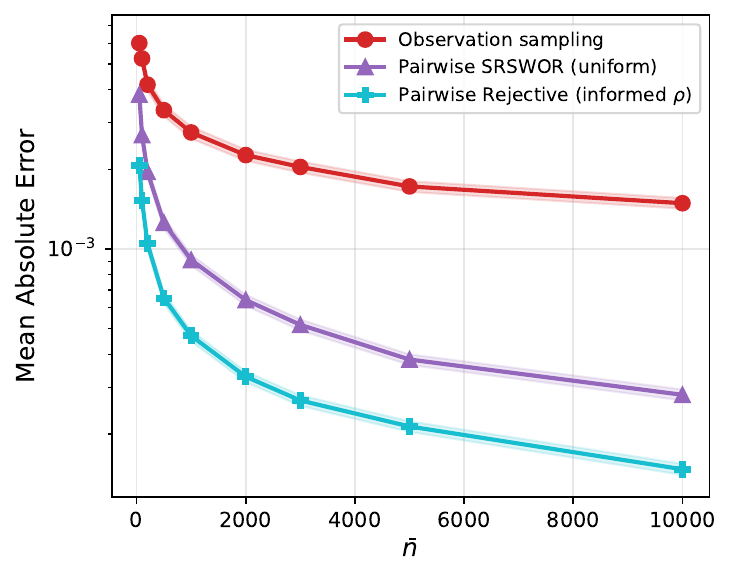}
        \caption{}
    \end{subfigure}\begin{subfigure}[t]{0.24\linewidth}
        \includegraphics[width=\linewidth]{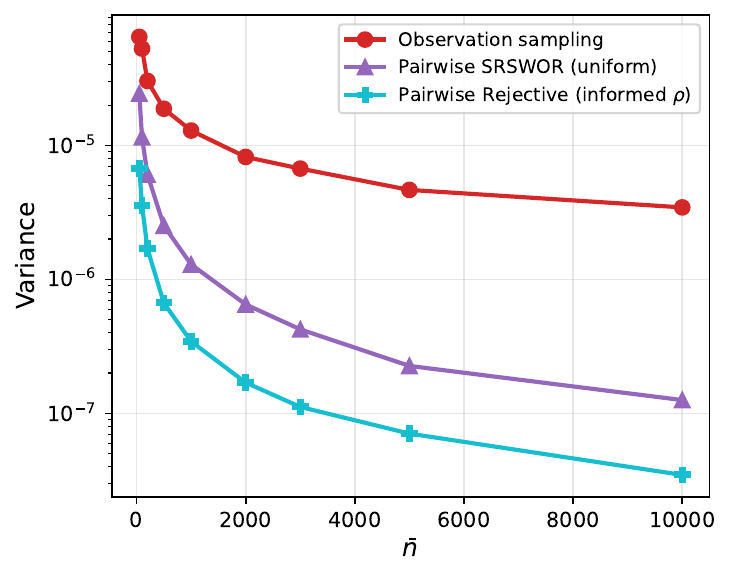}
        \caption{}
    \end{subfigure}
    \caption{MAE and variance comparison on Cora (a, b, c, d) and MovieLens (e, f, g, h) between Observation Sampling, Bernoulli pairwise sampling and Poisson informed pairwise sampling.}
    \label{fig:supp-cora-movielens}
\end{figure}

\begin{figure}[ht]
    \centering
    \begin{subfigure}[t]{0.24\linewidth}
        \includegraphics[width=\linewidth]{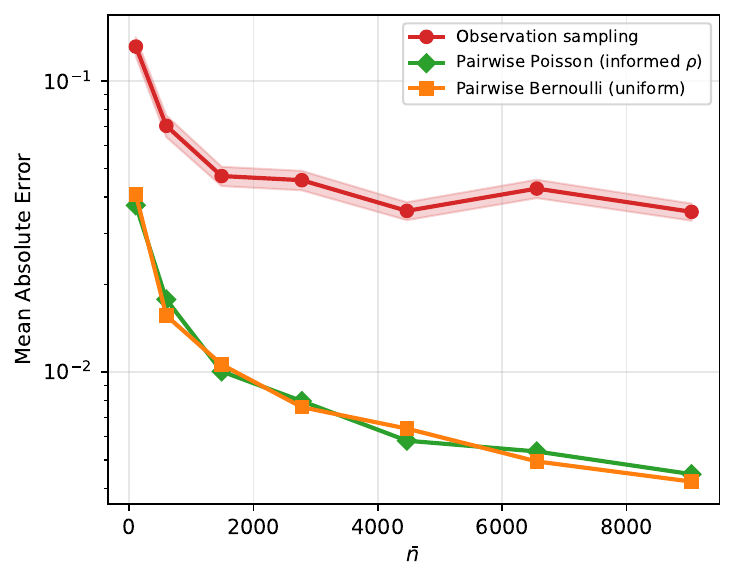}
        \caption{}
    \end{subfigure}\begin{subfigure}[t]{0.24\linewidth}
        \includegraphics[width=\linewidth]{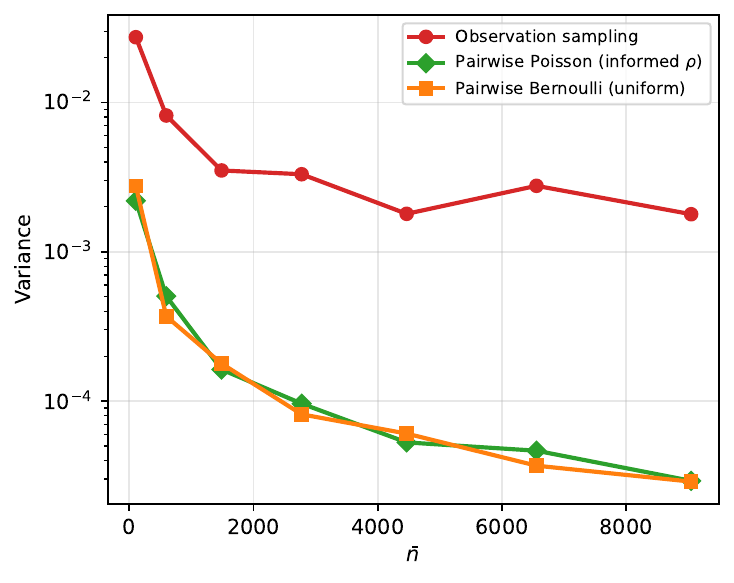}
        \caption{}
    \end{subfigure}\begin{subfigure}[t]{0.24\linewidth}
        \includegraphics[width=\linewidth]{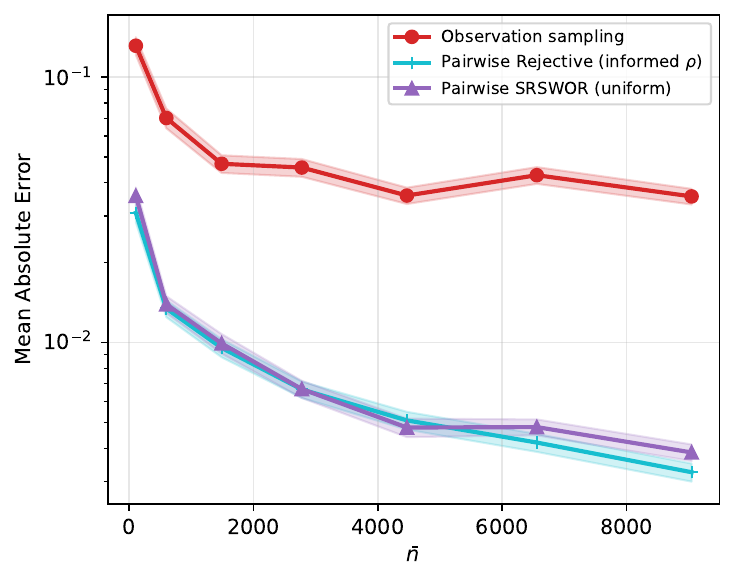}
        \caption{}
    \end{subfigure}\begin{subfigure}[t]{0.24\linewidth}
        \includegraphics[width=\linewidth]{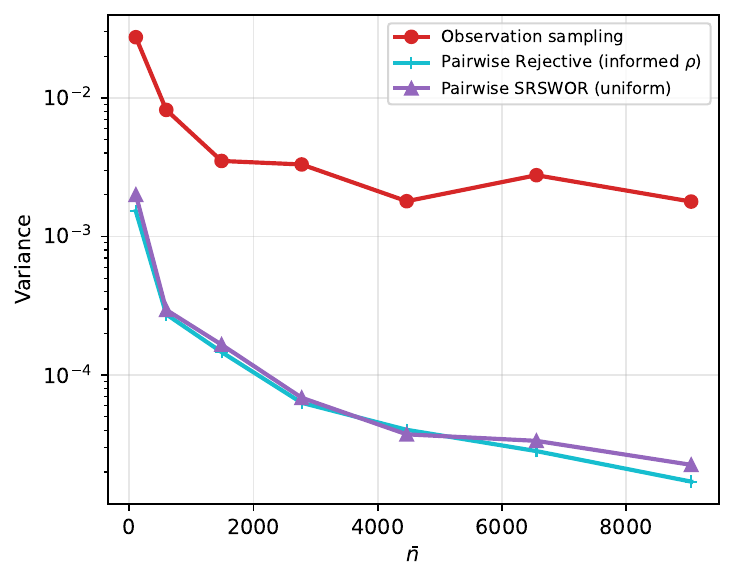}
        \caption{}
    \end{subfigure}\\
    \begin{subfigure}[t]{0.24\linewidth}
        \includegraphics[width=\linewidth]{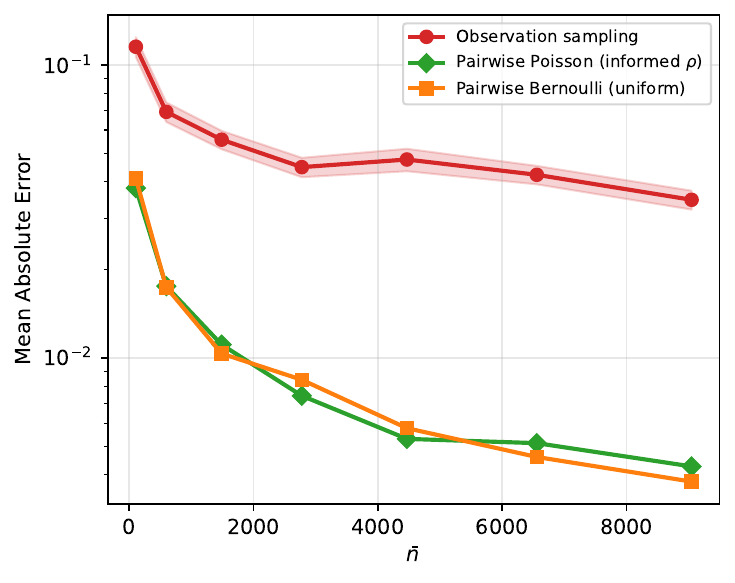}
        \caption{}
    \end{subfigure}\begin{subfigure}[t]{0.24\linewidth}
        \includegraphics[width=\linewidth]{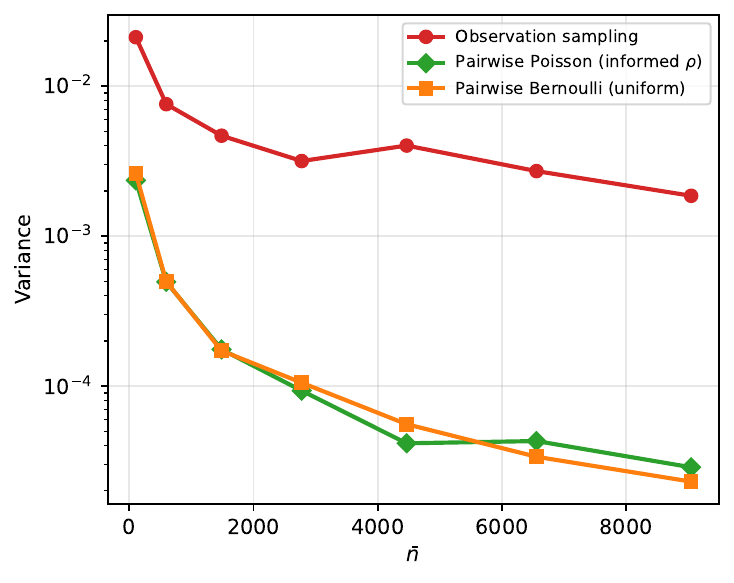}
        \caption{}
    \end{subfigure}\begin{subfigure}[t]{0.24\linewidth}
        \includegraphics[width=\linewidth]{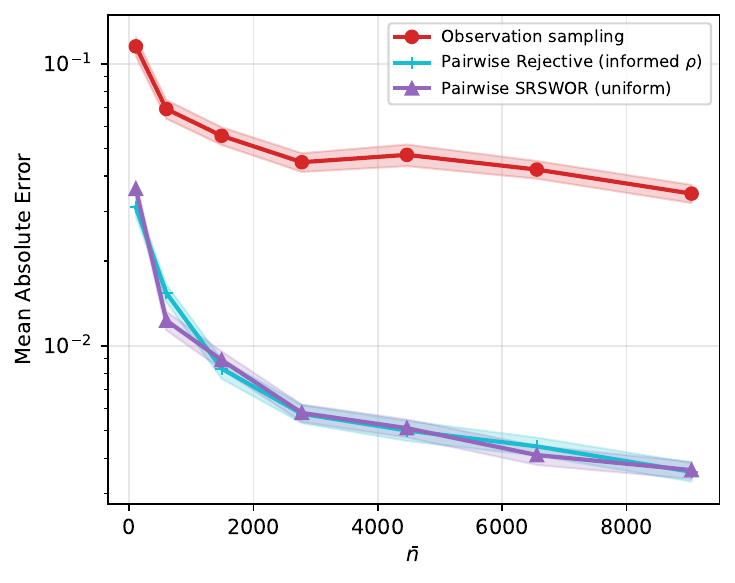}
        \caption{}
    \end{subfigure}\begin{subfigure}[t]{0.24\linewidth}
        \includegraphics[width=\linewidth]{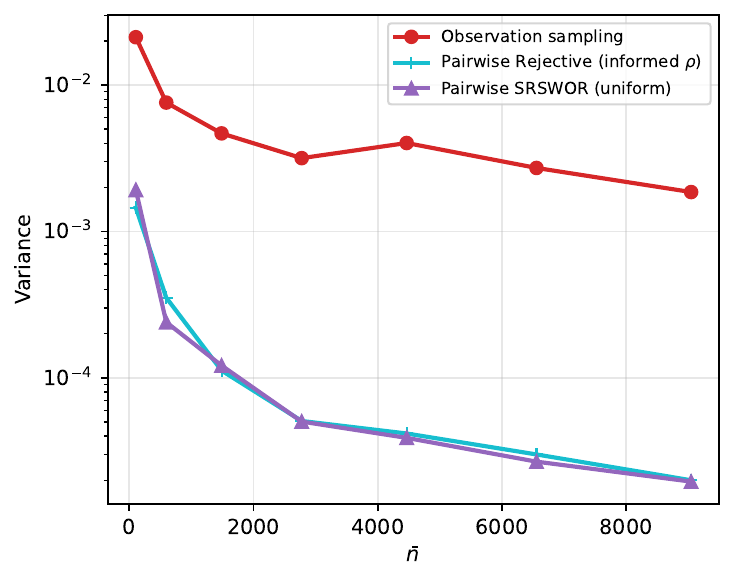}
        \caption{}
    \end{subfigure}\\
    \begin{subfigure}[t]{0.24\linewidth}
        \includegraphics[width=\linewidth]{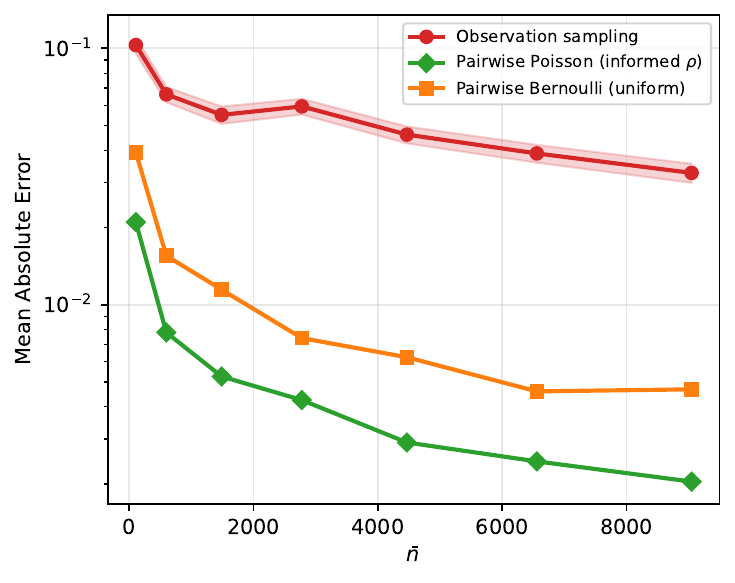}
        \caption{}
    \end{subfigure}\begin{subfigure}[t]{0.24\linewidth}
        \includegraphics[width=\linewidth]{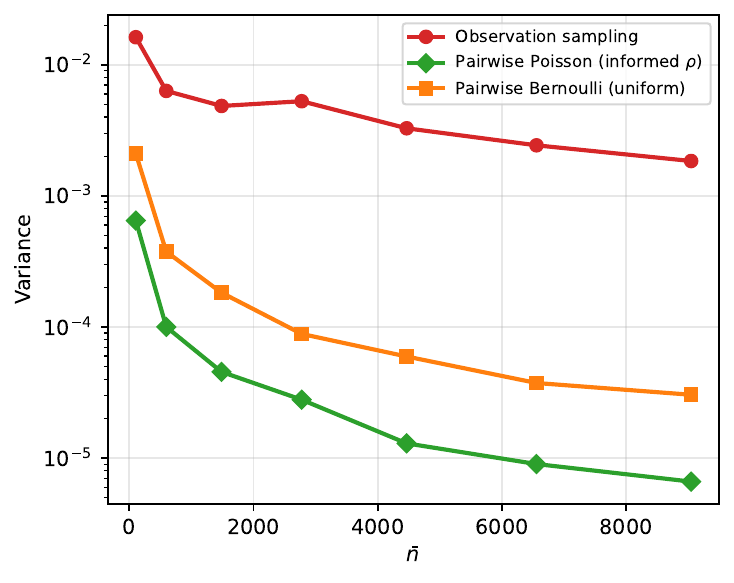}
        \caption{}
    \end{subfigure}\begin{subfigure}[t]{0.24\linewidth}
        \includegraphics[width=\linewidth]{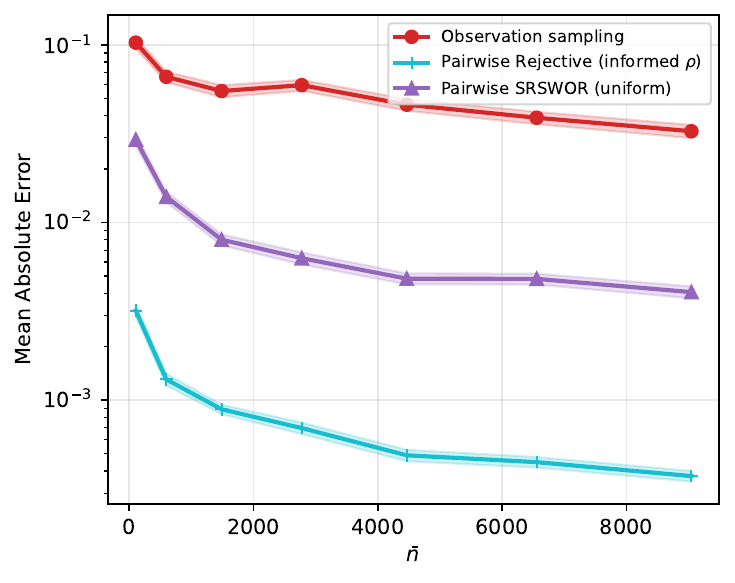}
        \caption{}
    \end{subfigure}\begin{subfigure}[t]{0.24\linewidth}
        \includegraphics[width=\linewidth]{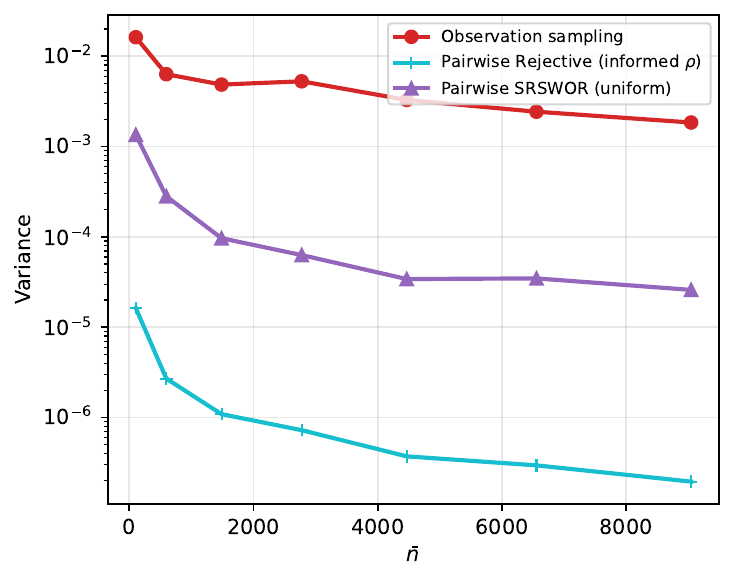}
        \caption{}
    \end{subfigure}
    \caption{MAE and variance comparison on LFW with the top-1 correlated attribute (a, b, c, d), with the top-3 correlated attributes (e, f, g, h) and with the \textit{ideal} target FAR as auxiliary information (i, j, k, l), between Observation Sampling, Bernoulli pairwise sampling and Poisson informed pairwise sampling.}
    \label{fig:supp-lfw}
\end{figure}

\paragraph{Node classification and face recognition with rejective sampling}

Figure \ref{fig:cora-learning-rej} shows that the rejective pairwise sampling strategy behaves similarly to Poisson sampling when compared with observation sampling, hard sampling, and its uninformed counterpart SRS.WOR. 

\begin{figure}
    \centering
    \begin{subfigure}[b]{0.32\textwidth} \includegraphics[width=\linewidth]{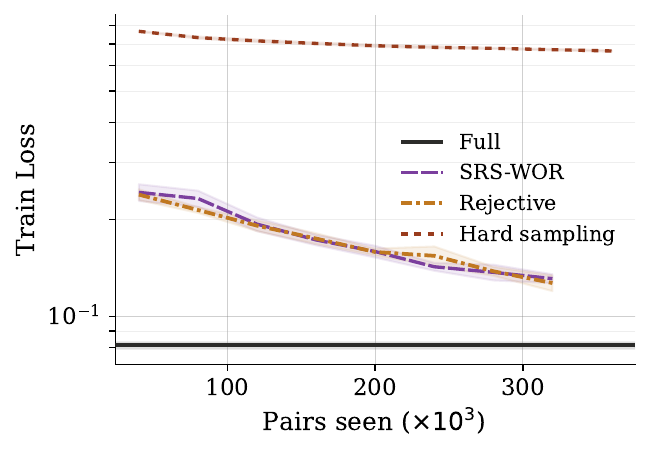}
        \caption{Cora}
        \label{fig:loss-cora-rej}
    \end{subfigure}
    \begin{subfigure}[b]{0.32\textwidth} \includegraphics[width=\linewidth]{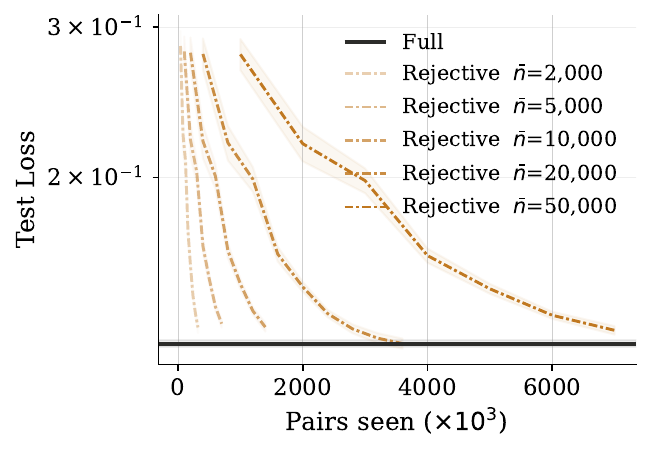}
        \caption{Cora}
        \label{fig:budget-cora-rej}
    \end{subfigure}
     \begin{subfigure}[b]{0.32\textwidth} \includegraphics[width=\linewidth]{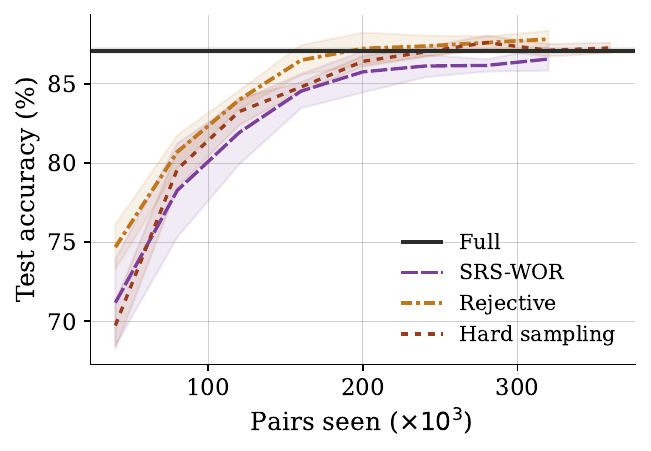}
        \caption{Cora}
        \label{fig:acc-cora-rej}
    \end{subfigure}\\
    \begin{subfigure}[b]{0.32\textwidth} \includegraphics[width=\linewidth]{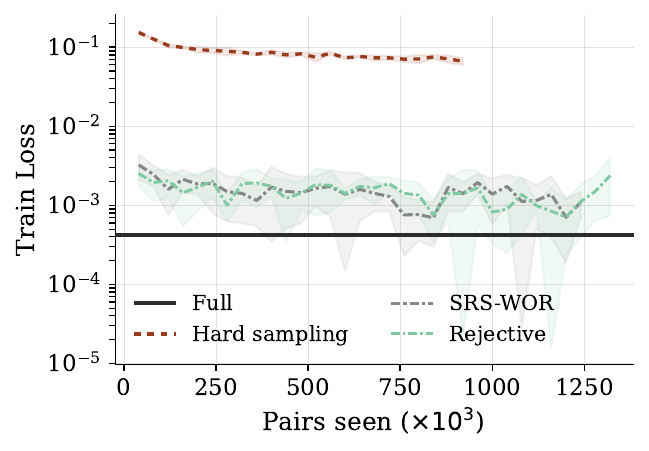}
        \caption{LFW}
        \label{fig:loss-lfw-rej}
    \end{subfigure}
    \begin{subfigure}[b]{0.32\textwidth} \includegraphics[width=\linewidth]{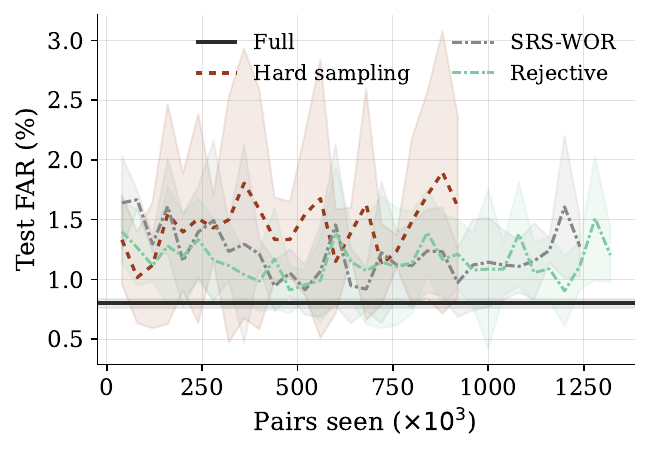}
        \caption{LFW}
        \label{fig:budget-lfw-rej}
    \end{subfigure}
     \begin{subfigure}[b]{0.32\textwidth} \includegraphics[width=\linewidth]{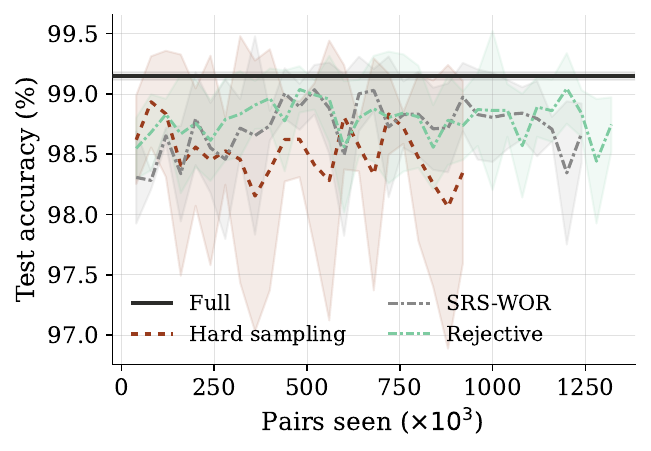}
        \caption{LFW}
        \label{fig:acc-lfw-rej}
    \end{subfigure}
    \caption{Loss curves on Cora and LFW (average over 5 seeds, shaded = $\pm$std). (a) At budget $\bar n= 2000$ per epoch (over 150 epochs). (b) Rejective across budgets. (c) Best test accuracy vs budget.}
    \label{fig:cora-learning-rej}
\end{figure}

\paragraph{Loss-based proxy as an idealistic oracle.}
Figure~\ref{fig:idealcaselfwlearning} shows results when the contrastive loss on the pretrained embeddings is used directly as the sampling proxy
$\rho$. Unlike hard sampling, inclusion probabilities remain tractable, so the HT correction is applied and the gradient step is unbiased. This configuration consistently outperforms hard sampling in terms of stability, final accuracy, and FAR, confirming
that debasing the loss estimator matters for learning, beyond the practical benefit of focusing on informative pairs.

\begin{figure}
    \centering
    \begin{subfigure}[b]{0.32\textwidth} \includegraphics[width=\linewidth]{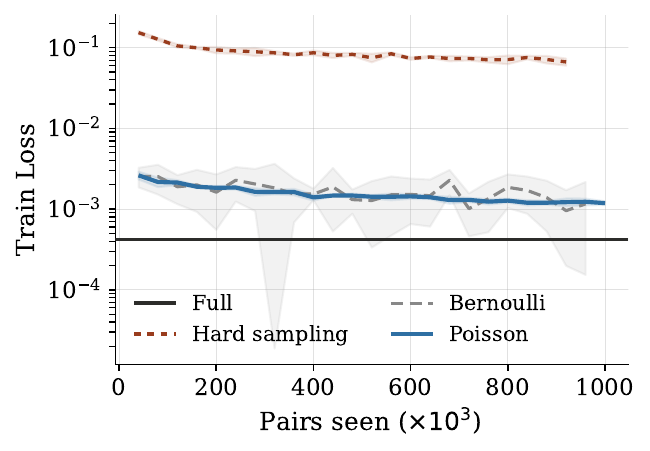}
        \caption{}
        \label{fig:ideal-loss-lfw}
    \end{subfigure}
    \begin{subfigure}[b]{0.32\textwidth} \includegraphics[width=\linewidth]{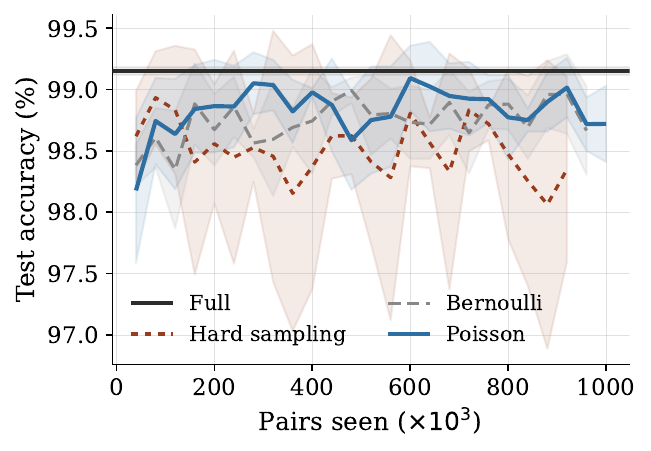}
        \caption{}
        \label{fig:ideal-acc-lfw}
    \end{subfigure}
    \begin{subfigure}[b]{0.32\textwidth} \includegraphics[width=\linewidth]{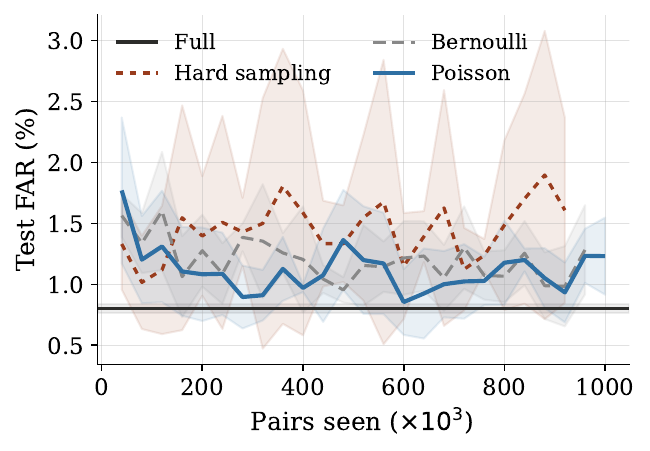}
        \caption{}
        \label{fig:ideal-far-lfw}
    \end{subfigure}\\
    \begin{subfigure}[b]{0.32\textwidth} \includegraphics[width=\linewidth]{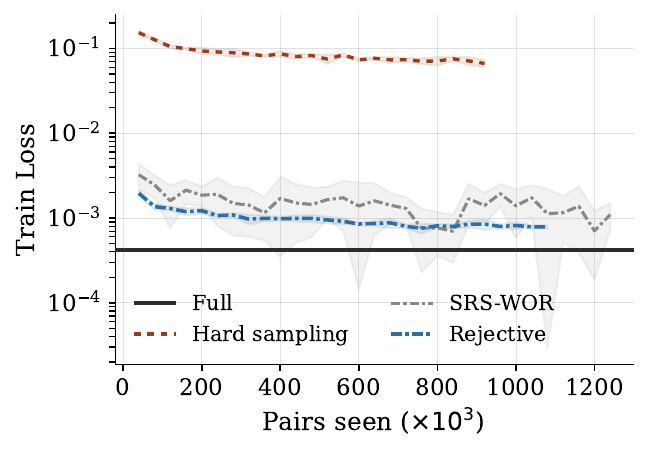}
        \caption{}
        \label{fig:ideal-loss-lfw-rej}
    \end{subfigure}
    \begin{subfigure}[b]{0.32\textwidth} \includegraphics[width=\linewidth]{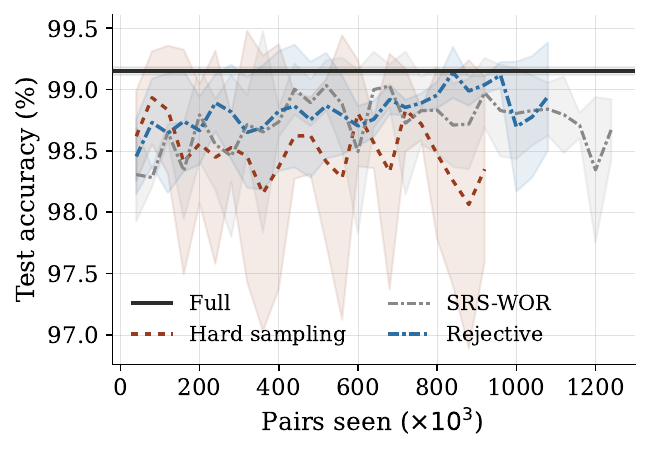}
        \caption{}
        \label{fig:ideal-acc-lfw-rej}
    \end{subfigure}
    \begin{subfigure}[b]{0.32\textwidth} \includegraphics[width=\linewidth]{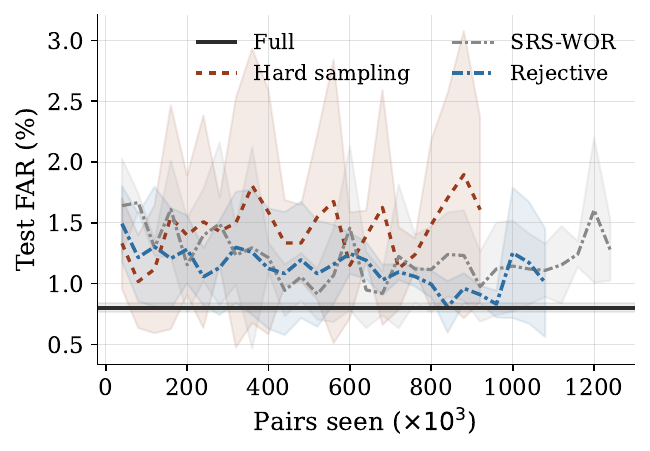}
        \caption{}
        \label{fig:ideal-far-lfw-rej}
    \end{subfigure}
    \caption{Example of directly using the Siamese loss at the beginning of training as auxiliary information to build the inclusion probabilities. The first line reports results for Poisson sampling, and the second one for Rejective sampling.}
    \label{fig:idealcaselfwlearning}
\end{figure}

\end{document}